\documentclass{article}

\usepackage[nonatbib,preprint]{neurips_2021}

\usepackage[utf8]{inputenc} 
\usepackage[T1]{fontenc}    
\usepackage{hyperref}       
\usepackage{url}            
\usepackage{booktabs}       
\usepackage{amsfonts}       
\usepackage{nicefrac}       
\usepackage{microtype}      
\usepackage{enumitem}
\usepackage{amsmath}
\usepackage{mathrsfs}
\usepackage{algorithm}
\usepackage{algorithmic}
\usepackage{bm}
\usepackage{xcolor}
\usepackage{graphicx}
\usepackage{amssymb}
\usepackage{cite}
\usepackage{subfigure}
\newtheorem{proof}{Proof}
\newtheorem{lemma}{Lemma}
 
\newtheorem{theorem}{Theorem} 
\newtheorem{assumption}{Assumption}
\newtheorem{definition}{Definition}

\title{Signal Transformer: Complex-valued Attention and Meta-Learning for Signal Recognition} 

\author{%
	Yihong Dong\\
	Tongji University\\
	\texttt{duh\_dyh@tongji.edu.cn}\\
	\And
	Ying Peng\\
	Tongji University\\
	\texttt{1853287@tongji.edu.cn}\\
	\And
	Muqiao Yang\\
	Carnegie Mellon University\\
	\texttt{muqiaoy@cs.cmu.edu}\\
	\And
	Songtao Lu\\
	IBM Thomas J. Watson Research Center\\
	\texttt{songtao@ibm.com}\\
	\And
	Qingjiang Shi\\
	Tongji University\\
	\texttt{shiqj@tongji.edu.cn}\\
}

\begin{document}
	
	\maketitle
	
	\begin{abstract}
		Deep neural networks have been shown as a class of useful tools for addressing signal recognition issues in recent years, especially for identifying the nonlinear feature structures of signals. However, this power of most deep learning techniques heavily relies on an abundant amount of training data, so the performance of classic neural nets decreases sharply when the number of training data samples is small or unseen data are presented in the testing phase. This calls for an advanced strategy, i.e., model-agnostic meta-learning (MAML), which is able to capture the invariant representation of the data samples or signals. In this paper, inspired by the special structure of the signal, i.e., real and imaginary parts consisted in practical time-series signals, we propose a Complex-valued Attentional MEta Learner (CAMEL) for the problem of few-shot signal recognition by leveraging attention and meta-learning in the complex domain. To the best of our knowledge, this is also the first complex-valued MAML that can find the first-order stationary points of general nonconvex problems with theoretical convergence guarantees. Extensive experiments results showcase the superiority of the proposed CAMEL compared with the state-of-the-art methods.
	\end{abstract}
	
	\section{Introduction}
	With the recent explosion of deep learning, signal recognition has made some remarkable advances \cite{o2016convolutional,NIPS2017_6f2268bd,9234613,9383106}. To achieve these, a large volume of data is required to obtain satisfactory performance. However, the deep learning models trained with traditional supervised learning methods often perform poorly or even fail when only a small amount of data is available or when they need to adapt to unseen tasks or time-varying ones. In practical signal recognition tasks, the collection and annotation of abundant data are notoriously expensive, especially for some rare but important signals. Another critical challenge is the presence of noise, because the signal data varies for different signal-to-noise ratios (SNRs), and in real-world scenarios, the deep neural networks (DNNs) have to adapt to real-time variations in SNRs. 
	
	Meta-learning technique \cite{finn2017model,NEURIPS2018_8e2c381d,yoon2018bayesian,zhang2018metagan,balaji2018metareg} seeks to resolve above challenges by learning how to learn like humans do. We know that humans can effectively utilize prior knowledge and experience to learn new skills rapidly with very few examples. Similarly, the meta-learner is trained on the distribution of homogeneous tasks, with the goal of learning internal features that are broadly applicable to all tasks, rather than a single individual task. Equipped with these sensitive internal features, the meta-learner is able to produce significant improvements of adaptation ability via fine-tuning. Recently, meta-learning has demonstrated promising performance in many fields \cite{NEURIPS2019_92262bf9,NEURIPS2019_fd0a5a5e,NEURIPS2019_6fe43269,NEURIPS2019_b2945042,NEURIPS2019_7a9a322c,NEURIPS2019_f4aa0dd9,NEURIPS2019_8c235f89,NEURIPS2019_072b030b,zhuang2020no,NEURIPS2020_3214a6d8,NEURIPS2020_0a716fe8,NEURIPS2020_1e04b969,NEURIPS2020_4b86ca48,NEURIPS2020_ee89223a,NEURIPS2020_731c83db,NEURIPS2020_cc3f5463,NEURIPS2020_84c578f2,NEURIPS2020_171ae1bb,NEURIPS2020_bdbd5ebf,NEURIPS2020_cfee3986}. Please see the supplementary material for detailed related work in Section \ref{related work}. However, for some particular fields, especially signal recognition, existing meta-learning methods generally neglect the prior knowledge of the signals, i.e., temporal information and complex domain information. For models with insufficient training data, it is crucial to incorporate this prior knowledge.
	
	As such, we take into account the attention mechanism \cite{NIPS2017_3f5ee243} and the complex-valued neural network \cite{DBLP:conf/iclr/TrabelsiBZSSSMR18, hirose2012complex, tu2020complex} for signal recognition, respectively. Attention mechanisms have been widely adopted in many time series learning tasks, such as natural language processing. It became an integral component of Recurrent neural networks (RNNs), long short-term memory \cite{hochreiter1997long} and gated recurrent \cite{chung2014empirical} neural networks, until Transformer \cite{NIPS2017_3f5ee243} was proposed. Since then, self-attention is able to replace RNN with better performance and parallel computation. Therefore, we adapt the attention mechanism to the signal recognition task. Since the signals contain both magnitude and phase, complex numbers are used for the representation of signals. Consequently, complex arithmetic operations are the essential part of signal processing. Intuitively, complex-valued neural networks should be built to address the signal recognition problem. However, to the best of our knowledge, the meta-learning method equipped with attention mechanisms in the complex-valued neural networks has not been investigated. 
	
	In this paper, we propose a Complex-valued Attentional MEta Learner (CAMEL), for few-shot signal recognition, which generalizes meta-learning and attention to the complex domain. With the help of these novel designs, CAMEL has succeeded in capturing more information from the signal data. The prior knowledge assists CAMEL in preventing overfitting and improving its performance. For better understanding the proposed architecture, the overview of CAMEL is illustrated in Figure \ref{overview}. Notice that CAMEL can be applied to any kind of complex-valued data. By leveraging existing meta-learning and few learning methods in extensive experiments, the proposed method shows consistently better performance compared with the state-of-the-art methods. The effectiveness of each novel component in CAMEL is verified via ablation studies. From the convergence analysis of 
	complex-valued MAML, it is shown that CAMEL is able to find an $\epsilon$ first-order stationary point for any positive $\epsilon$ after at most $O\left(1/\epsilon^2\right)$ iterations with second-order information.
	
	\begin{figure}[t]
		\centering
		\includegraphics[width= 10cm]{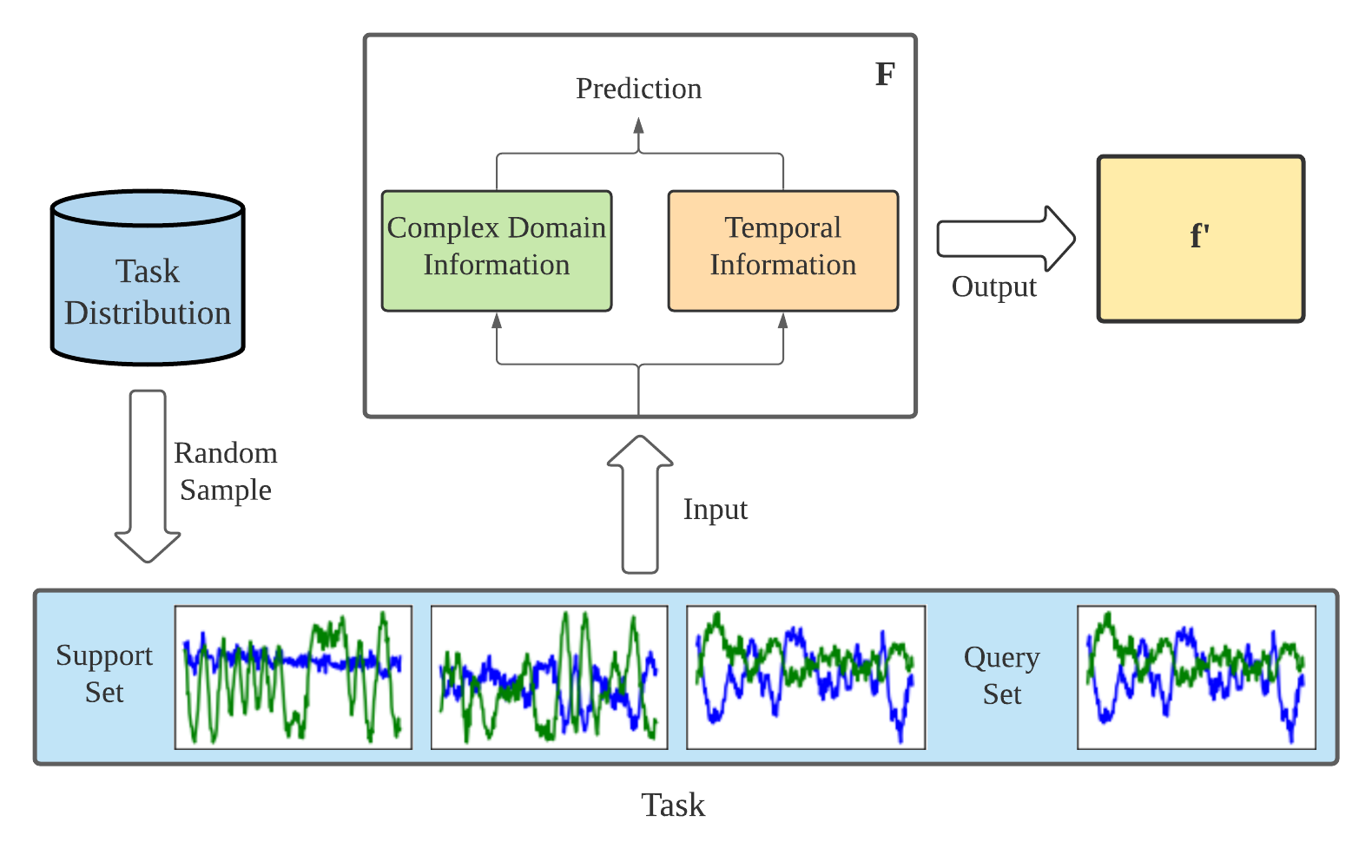}
		\caption{The overview of CAMEL.}
		\label{overview}
	\end{figure}
	
	
	The code of this paper will be released upon acceptance. Please see the supplementary material for notations, detailed derivation of Lemma, and more experiment results.
	

	\section{Motivation}
	Meta-learning is one of the most suited techniques to solve signal recognition problems, because, in the real world, signal annotation is expensive and models need to adapt to changing SNRs, whereas meta-learning has an explicit goal of fast adaptation. To further improve the effectiveness of meta-learning in applications to signal processing, we consider incorporating prior knowledge of signal data to the model by CAMEL that can generalize meta-learning with attention to the complex domain so that we are able to extract complex domain and temporal information from signal data.
	
	However, lots of so called complex-valued neural networks treat a complex number as two real numbers, i.e., real and imaginary parts of the complex number, and design special network structures to recover complex operations using these real numbers. We refer to these special complex-valued neural networks as in-phase/quadrature complex-valued neural networks (IQCVNNs). Although IQCVNNs can deal with complex-valued problems, essentially the neural nets are still working with real-valued ones, since IQCVNNs work without defining complex derivatives and the complex chain rules in back-propagation. We refer to the complex-valued neural networks that define complex derivatives and the complex chain rules as complex derivatives complex-valued neural networks (CDCVNNs). It turns out that compared with IQCVNNs, CDCVNNs can perform complex operations with fewer parameters. To be more specific, we give the following lemmas to show the significance of CDVNNs compared with IQCVNNs with respect to time complexity.
	\begin{lemma} \label{lemma1}
		If a function $g$ is complex analytic, the time complexity of the derivative of $g$ in IQCVNNs are twice that of the complex derivative of $g$ in CDCVNNs. 
	\end{lemma}
	\begin{lemma}  \label{lemma2}
		The complex-valued convolutional layer and complex-valued fully connected layer is complex analytic.
	\end{lemma}
	As we know, the convolutional and fully connected layers are the most computationally intensive parts of a neural network. Therefore, although it has a similar effect to the complex-valued neural network, IQCVNNs far exceed the CDCVNNs in terms of the time complexity of back-propagation. Especially, meta-learning requires second-order information of the objective function to guarantee convergence \cite{fallah2020convergence}, which forces us to implement CDCVNNs. The complex chain rule is a key to implementing CDCVNNs. According to the complex chain rule, we are able to derive the \emph{outer-loop update process} of CAMEL, which is different from that of MAML.
	
	Complex-valued attention is also necessary for CAMEL to obtain the temporal information from signal data. However, in complex-valued attention, it is required to compute the derivative of the mapping from complex to real domain since calculating the similarity coefficient between two pairs leads to the real numbers in the activation function of the complex-valued neural nets. Given the following lemma, we know that the derivative of the function will be non-analytic, since constant function is useless in identifying the features of data.
	\begin{lemma}\label{cr}
		$\forall g: \mathbb{C} \rightarrow \mathbb{R}$, $g$ is analytic if and only if $g$ is a constant function.
	\end{lemma}
	
	To the best of our knowledge, attention in the complex domain has rarely been studied. \footnote{A closely related work is \cite{yang2020complex}, which proposed a complex transformer and developed attention and encoder-decoder network operating for complex input. However, they utilized eight attentions to represent complex-valued attention without considering the nonlinear components of attention such as softmax and activation functions, etc.}. Therefore, we here study complex-valued attention and propose CAMEL as presented in the next section.

	\section{CAMEL}
	Please see the supplementary material for the definitions of complex derivative, analytic function, and the Cauchy-Riemann equations in Section \ref{Definition Recall}. 

	\subsection{Algorithm Design}
	\label{CAMEL}
	CAMEL utilizes complex-valued neural networks and attention to provide prior knowledge, i.e., complex domain and temporal information, to prevent overfitting during training. It resembles its namesake animal, camel, which stores water and nutrients with its hump to ensure its survival in extreme conditions.
	
	CAMEL updates parameters through back-propagation by the chain rule. However, traditional chain rule does not work, because CAMEL is non-analytic. 
	
	\textbf{The chain rule for complex variables} \ The chain rule is different when the function is non-analytic. For a non-analytic composite function $g(\mathbf{u})$, where $\mathbf{u} = h(\mathbf{x})$, we can apply the following chain rule:
	\begin{equation} \label{cvcr}
		\dfrac{\partial g(\mathbf{u})}{\partial \mathbf{x}} =  \dfrac{\partial g(\mathbf{u})}{\partial \mathbf{u}} \dfrac{\partial \mathbf{u}}{\partial \mathbf{x}} + \dfrac{\partial g(\mathbf{u})}{\partial \mathbf{u}^*} \dfrac{\partial \mathbf{u}^*}{\partial \mathbf{x}}
	\end{equation}
	where $g$ is a continuous function and $\mathbf{u}^*$ denotes the conjugate vector of $\mathbf{u}$. Note that if the function is analytic, the second term equals zero and \eqref{cvcr} turns into the normal chain rule. In the case of matrix derivatives, the chain rule can be written as:
	\begin{equation} \label{cvcr2}
		\dfrac{\partial g(\mathbf{U})}{\partial \mathbf{X}} =  \dfrac{\partial \text{Tr}\left( \left(\dfrac{\partial g(\mathbf{U})}{\partial \mathbf{U}}\right)^T\partial \mathbf{U}\right)}{\partial \mathbf{X}}+ \dfrac{\partial \text{Tr}\left( \left(\dfrac{\partial g(\mathbf{U})}{\partial \mathbf{U^*}}\right)^T\partial \mathbf{U^*}\right)}{\partial \mathbf{X}}
	\end{equation}
	where $\mathbf{U} = h(\mathbf{X})$ is non-analytic, $\mathbf{U}$ and $\mathbf{X}$ are two complex matrices, and $\left(\cdot \right)^T$ denotes the transpose of a matrix.
	
	Under \eqref{cvcr} and \eqref{cvcr2}, CAMEL is able to update the parameters as expected. Formally, we define the base model of CAMEL to be a complex-valued attentional neural network with meta-parameters $\bm{\theta} \in \mathbb{C}$. The goal is to learn a sensitive initial $\bm{\theta}$, for which the network performs well on the $i$th query set $Q_i$ after few gradient update steps on the $i$th support set $S_i$ to obtain $\bm{\theta}'_i$.
	Here, $\mathcal{T}_i = \{S_i, Q_i\}$ is a task randomly sampled from the task probability distribution $p(\mathcal{T})$. The update steps above are termed as the \emph{inner-loop update process}, which can be represented as:
	
	\begin{equation}\label{g}
		\bm{\theta}'_i = \bm{\theta} - \alpha \nabla_{\bm{\theta}} \mathcal{L}_{S_i}\left(\bm{\theta}\right)
	\end{equation}
	where $\alpha$ is a learning rate and $\nabla_{\bm{\theta}} \mathcal{L}_{S_i}\left(\bm{\theta}\right)$ denotes the gradient on the support set of task $i$. The meta-parameters $\bm{\theta}$ are trained by optimizing the performance of $\bm{\theta}'_i$. Consequently, the \emph{meta-objective} is defined as follows: 
	\begin{equation}
		\begin{aligned} \label{lmeta}
			\min_{\bm{\theta}} \mathcal{L}_{meta}(\bm{\theta})  \triangleq  \mathbb{E}_{\{S_i, Q_i\} \sim p(\mathcal{T})} \left[ \mathcal{L}_{Q_i}(\bm{\theta}'_i)\right]
		\end{aligned}
	\end{equation}
	where $\mathcal{L}_{Q_i}(\bm{\theta}'_i)$ denotes the loss on the query set of task $i$ after the inner-loop update process. 
	As the underlying $p(\mathcal{T})$ is unknown, evaluation of the expectation in the right hand side of \eqref{lmeta} is often computationally prohibitive. Therefore, we can minimize the function $ \mathcal{L}_{meta}(\bm{\theta})$ with a batch of tasks $\left\{\mathcal{T}_i\right\}^B_{i=1}$ that are independently drawn from $p(\mathcal{T})$, which can be expressed as:
	\begin{equation}\label{7}
		\begin{aligned}
			\mathcal{L}_{meta}(\bm{\theta}) & = \dfrac{1}{B} \Sigma_{\{S_i, Q_i\} \sim p(\mathcal{T})} \  \mathcal{L}_{Q_i}(\bm{\theta}'_i)\\
			& = \dfrac{1}{B} \Sigma_{\{S_i, Q_i\} \sim p(\mathcal{T})} \ \mathcal{L}_{Q_i}\left(\bm{\theta} - \alpha \Sigma_j \nabla_{\bm{\theta}'_i} \mathcal{L}_{S_i}\left(\bm{\theta}'_i\right)\right).
		\end{aligned}
	\end{equation}
	The optimization of the meta-objective is referred to as the \emph{outer-loop update process}, which can be expressed as:
	\begin{equation} \label{ou}
		\bm{\theta} \leftarrow \bm{\theta} - \beta \nabla_{\bm{\theta}} \mathcal{L}_{meta}(\bm{\theta})
	\end{equation}
	where $\beta$ denotes the meta learning rate. Define 
	\begin{equation} \label{H}
		\mathbf{H}^i_{\bm{\theta} \bm{\theta}} \triangleq \left(\dfrac{\partial \nabla_{\bm{\theta}}  \mathcal{L}_{S_i} \left(\bm{\theta}\right)}{\partial \bm{\theta}}\right)^*, \qquad 
		\mathbf{H}^i_{\bm{\theta}^* \bm{\theta}} \triangleq \left(\dfrac{\partial \nabla_{\bm{\theta}^*}  \mathcal{L}_{S_i} \left(\bm{\theta}\right)}{\partial \bm{\theta}}\right)^*.
	\end{equation}
	\begin{lemma} \label{lemma4}
		In response to complex meta-parameters $\bm{\theta}$, we have
		\begin{equation}\label{lq}
			\nabla_{\bm{\theta}} \mathcal{L}_{meta}(\bm{\theta}) = \dfrac{1}{B} \Sigma_{\{S_i, Q_i\} \sim p(\mathcal{T})}  \left( \mathbf{I} - \alpha \mathbf{H}^i_{\bm{\theta} \bm{\theta}} \right) \nabla_{\bm{\theta}'_i} \mathcal{L}_{Q_i}(\bm{\theta}'_i) - \alpha  \mathbf{H}^i_{\bm{\theta}^* \bm{\theta}} \nabla_{\bm{(\bm{\theta}'_i)^*}} \mathcal{L}_{Q_i}(\bm{\theta}'_i).
		\end{equation}
	\end{lemma}
	
	Then, according to \eqref{lq}, the \emph{outer-loop update process} for complex meta-parameters $\bm{\theta}$ is
	\begin{equation}
		\bm{\theta} \leftarrow \bm{\theta} - \beta \dfrac{1}{B} \Sigma_{\{S_i, Q_i\} \sim p(\mathcal{T})} \left( \mathbf{I} - \alpha \mathbf{H}^i_{\bm{\theta} \bm{\theta}} \right) \nabla_{\bm{\theta}'_i} \mathcal{L}_{Q_i}(\bm{\theta}'_i) - \alpha  \mathbf{H}^i_{\bm{\theta}^* \bm{\theta}} \nabla_{\bm{(\bm{\theta}'_i)^*}} \mathcal{L}_{Q_i}(\bm{\theta}'_i).
	\end{equation}
	The complete algorithm description of is outlined in Algorithm \ref{alg:training}.
	
	\begin{algorithm}[htb]
		\caption{Pseudocode for CAMEL Update}
		\label{alg:training}
		\begin{algorithmic}
			\REQUIRE The distribution over tasks $p(\mathcal{T})$.
			\REQUIRE The learning rates $\alpha, \beta$.
			\ENSURE The meta-parameters $\bm{\theta}$ of CAMEL.
			\STATE Randomly initialize the meta-parameters $\bm{\theta}$ of CAMEL.
			\REPEAT
			\STATE Sample batch of tasks $\mathcal{T}_i = \{S_i, Q_i\} \sim p(\mathcal{T})$
			\FOR{each $\{S_i, Q_i\}$}
			\STATE Evaluate $\nabla_{\bm{\theta}_i} \mathcal{L}_{S_i}\left(\bm{\theta}_i\right)$ via the complex chain rule \eqref{cvcr} and \eqref{cvcr2}.
			\STATE Update $\bm{\theta}'_i = \bm{\theta} - \alpha \nabla_{\bm{\theta}} \mathcal{L}_{S_i}\left(\bm{\theta}\right)$.
			\STATE Set $\mathbf{H}^i_{\bm{\theta} \bm{\theta}} = \left(\dfrac{\partial \nabla_{\bm{\theta}}  \mathcal{L}_{S_i} \left(\bm{\theta}\right)}{\partial \bm{\theta}}\right)^*$ and $\mathbf{H}^i_{\bm{\theta}^* \bm{\theta}} = \left(\dfrac{\partial \nabla_{\bm{\theta}^*}  \mathcal{L}_{S_i} \left(\bm{\theta}\right)}{\partial \bm{\theta}}\right)^*$
			\STATE Evaluate $\nabla_{\bm{\theta}'_i} \mathcal{L}_{Q_i}(\bm{\theta}'_i)$ and $\nabla_{\bm{(\bm{\theta}'_i)^*}} \mathcal{L}_{Q_i}(\bm{\theta}'_i)$ via the complex chain rule \eqref{cvcr} and \eqref{cvcr2}.
			\ENDFOR
			\STATE Update $\bm{\theta} \leftarrow \bm{\theta} - \beta \dfrac{1}{B} \Sigma_{\{S_i, Q_i\} \sim p(\mathcal{T})} \left( \mathbf{I} - \alpha \mathbf{H}^i_{\bm{\theta} \bm{\theta}} \right) \nabla_{\bm{\theta}'_i} \mathcal{L}_{Q_i}(\bm{\theta}'_i) - \alpha  \mathbf{H}^i_{\bm{\theta}^* \bm{\theta}} \nabla_{\bm{(\bm{\theta}'_i)^*}} \mathcal{L}_{Q_i}(\bm{\theta}'_i)$
			\UNTIL{convergence}
		\end{algorithmic}
	\end{algorithm}
	
	\subsection{Complex-valued Attention} 
	\label{CVA} 
	The attention mechanisms are widely used in various areas of deep learning, but attention for the complex domain have rarely been addressed. A significant reason is that the attention has to utilize the softmax function to calculate the similarity coefficient, which must be real numbers rather than complex numbers. 
	According to Lemma \ref{cr}, it is a constant function or a non-analytic function. However, the constant functions are useless and discardable in neural networks, while non-analytic functions cannot be derived at arbitrary points in complex domain. As a result, we had to utilize the complex gradient vector.
	
	\textbf{Complex gradient vector} \ If $\hat{g}$ is the real function of a complex vector $\mathbf{x}$, then the complex gradient vector is given by \cite{hjorungnes2011complex}:
	\begin{equation}
		\begin{aligned}\label{cgv}
			\nabla \hat{g}(\mathbf{x}) &= 2\dfrac{d \hat{g}(\mathbf{x})}{d \mathbf{x}^*}=\dfrac{d \hat{g}(\mathbf{x})}{\partial \Re(\mathbf{x}^*)} + j \dfrac{d \hat{g}(\mathbf{x})}{\partial \Im(\mathbf{x}^*)}.
		\end{aligned}
	\end{equation}
	
	\textbf{Complex-valued softmax function} Under \eqref{cgv}, we are able to define the generalized complex-valued softmax function as:
	\begin{equation}
		\textrm{C}_{sf}(\mathbf{x}) = \textrm{R}_{sf}(\hat{g}(\mathbf{x}))
	\end{equation}
	where $\textrm{R}_{sf}(\cdot)$ denotes the softmax function in real case and $\hat{g}(\cdot)$ denotes any function that maps complex numbers to real numbers, such as $abs(\cdot)$ (i.e., the magnitude of the complex numbers), $\Re(\cdot)$, and $\Im(\cdot)$, etc.
	
	Given a complex matrix $\mathbf{X}$, we can compute the complex matrix $\mathbf{Q}$, $\mathbf{K}$ and $\mathbf{V}$ using linear transformations, which are similar to complex-valued fully connected layers. Then the complex-valued attention can be written as: 
	\begin{equation}
		\begin{aligned}
			&\qquad \textrm{C}_{a}(\mathbf{Q},\mathbf{K},\mathbf{V}) = \textrm{C}_{sf} \left(\dfrac{\mathbf{Q}\mathbf{K}^T}{\sqrt{d_k}}\right)\mathbf{V}\\
			&=\textrm{C}_{sf}\left(\dfrac{(\Re(\mathbf{Q})\Re(\mathbf{K})^T-\Im(\mathbf{Q})\Im(\mathbf{K})^T)+j(\Re(\mathbf{Q})\Im(\mathbf{K})^T+\Im(\mathbf{Q})\Re(\mathbf{K})^T)}{\sqrt{d_k}}\right)\Re(\mathbf{V})\\
			& \quad + j \ \textrm{C}_{sf}\left(\dfrac{(\Re(\mathbf{Q})\Re(\mathbf{K})^T-\Im(\mathbf{Q})\Im(\mathbf{K})^T)+j(\Re(\mathbf{Q})\Im(\mathbf{K})^T+\Im(\mathbf{Q})\Re(\mathbf{K})^T)}{\sqrt{d_k}}\right)\Im(\mathbf{V})
		\end{aligned}
	\end{equation}
	where $\textrm{C}_{sf}(\cdot)$ acts on each row of the matrix and $d_k$ denotes the row dimension of $\mathbf{K}$ i.e. scaling factor.
	
	\textbf{Complex-valued multi-head attention} Complex-valued multi-headed attention allows models to jointly focus on information from different representations. 
	\begin{equation}
		\textrm{C}_{mha}(\mathbf{Q},\mathbf{K},\mathbf{V}) = Concat(\{\textrm{C}_{a}(\mathbf{QW^Q_k},\mathbf{KW^K_k},\mathbf{VW^V_k})\}^n_{k=1})\mathbf{W^O}
	\end{equation}
	where $\mathbf{W^Q_k}, \mathbf{W^K_k}, \mathbf{W^V_k}$, and $\mathbf{W^O}$ are the projection matrices and $Concat(\cdot)$ denotes the concatenation of inputs matrices.
	
	\textbf{Complex-valued normalization} Normalization, such as batch normalization \cite{ioffe2015batch} and layer normalization \cite{ba2016layer}, is an important component of neural networks. Especially, the batch normalization is commonly employed. However, for a complex vector $\mathbf{x}$, its variance, which has to be computed in normalization, is real. According to Lemma \ref{cr}, the variance is non-analytic. Therefore, in the back-propagation of complex-valued normalization, we have to utilize the complex gradient vector \eqref{cgv}.
	Define $\gamma$ as the complex scaling parameters and $\kappa$ as the complex shift parameters, the complex-valued normalization can be expressed as:
	\begin{equation}
		\begin{aligned}
			& \textrm{C}_{n}(\mathbf{x}) = \gamma \left( \text{Var}[\mathbf{x}]\right)^{-\frac{1}{2}} \left(\mathbf{x}-  \mathbb{E}[\mathbf{x}]\right) + \kappa\\
			&\text{Var}[\mathbf{x}] =  \mathbb{E}\{[\mathbf{x}-\mathbb{E}[\mathbf{x}]][\mathbf{x} -\mathbb{E}[\mathbf{x}]]^H\}
		\end{aligned}
	\end{equation}
	where $\mathbb{E}[\cdot]$ and $\text{Var}[\cdot]$ denote the expectation and variance, respectively, and $[\mathbf{x}]^H$ denotes the conjugate transpose of $\mathbf{x}$.
	
	\begin{figure}[t]
		\centering
		\includegraphics[width= 13.5 cm]{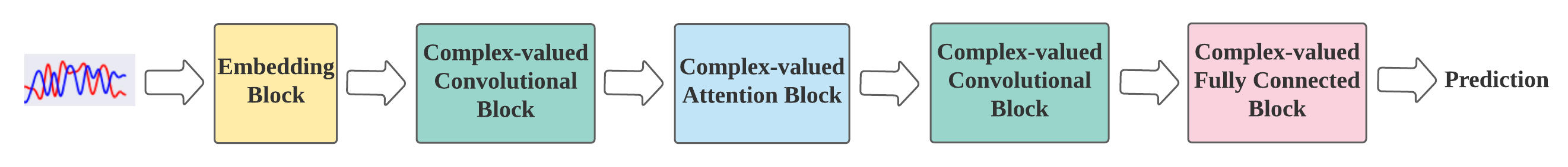}
		\caption{The architecture of CAMEL. The embedding block contains $1 \times 1$ complex-valued convolutional layers. The complex-valued convolutional block contains $3 \times 1$ complex-valued convolutional layer, complex-valued batch normalization, and complex-valued ReLU. The complex-valued attention block contains complex-valued attentions. The complex-valued fully connected block contains complex-valued fully connected layers, complex-valued batch normalization, and complex-valued ReLU.}
		\label{architecture}
	\end{figure}
	
	\textbf{Complex-valued activation function} \ The activation function is nonlinear, so that it is scarcely to be analytic. Most of the well-known activation functions are not analytic in the complex domain, such as Sigmoid, Tanh, and ReLU \cite{goodfellow2016deep}, etc. 
	Especially, the complex Sigmoid and Tanh is not bounded while in complex ReLU the complex numbers cannot be compared with zero.
	To this end, the complex-valued activation function can be defined as:
	\begin{equation} \label{AF}
		\textrm{C}_{af}(\mathbf{x}) = \textrm{R}_{af}(\Re(\mathbf{x})) + j \ \textrm{R}_{af}(\Im(\mathbf{x}))
	\end{equation}
	where $\textrm{R}_{af}(\cdot)$ denotes the activation function in real case. In this way, the $\textrm{C}_{Sigmoid}$ and $\textrm{C}_{Tanh}$ are bounded because the real and imaginary parts of them are bounded. Meanwhile, the complex $\textrm{C}_{ReLU}$ can be compared with zero because the real and imaginary parts of inputs can be compared with zero. However, since the complex-valued activation functions defined above are non-analytic in most cases, the complex chain rule is required for derivatives. 
	
	Please see the supplementary material for detailed complex-valued convolutional layer and complex-valued fully connected layer in Section \ref{CVNN}.
	
	\section{Convergence of CAMEL}
	In this section, we will show the convergence behavior of complex-valued MAML by following the previous work \cite{fallah2020convergence} in proving the convergence MAML in the real domain. To prove the complex-valued MAML, we need to utilize twice continuously differentiable, $L_i$-smooth, $\rho_i$-Lipschitz  continuous, and Hessian, etc. in complex domain. Please see the supplementary material for detailed Assumptions, Lemma, and proof of Theorem 1 in Section \ref{Proof of Theorem 1}. 
	\begin{theorem} 
		Suppose that Assumptions \ref{assumption 1}-\ref{assumption 5} hold and $\alpha\in(0, \frac{1}{6L}]$. Consider running complex-valued MAML with batch sizes $B\geq 20$ and $|Q|\geq \lceil 2\alpha^2\sigma_H^2 \rceil$. Following the definition in Lemma \ref{5}, let $\beta_k=\tilde{\beta}(\bm{\theta}_k)/12$. Then for any $\epsilon > 0$, complex-valued MAML finds a solution that 
		\begin{equation} \label{e}
			\mathbb{E}[||\nabla \mathcal{L}_{meta}(\bm{\theta}_\epsilon)||] \leq \max{\left\{ \sqrt{61(1+\frac{\rho\alpha}{L}\sigma)(\frac{\sigma^2}{B} + \frac{\tilde{\sigma}^2}{B|Q|} + \frac{\tilde{\sigma}^2}{|S|})}, \frac{61\rho\alpha}{L}(\frac{\sigma^2}{B} + \frac{\tilde{\sigma}^2}{B|Q|} + \frac{\tilde{\sigma}^2}{|S|}), \epsilon \right\}}
		\end{equation}
		after running for
		\begin{equation}
			\mathcal{O}(1)\Delta\min\left\{ \frac{L+\rho\alpha(\sigma+\epsilon)}{\epsilon^2}, \frac{LB}{\sigma^2} + \frac{L(B|Q|+|S|)}{\tilde{\sigma}^2} \right\}
		\end{equation}
		iterations, where $\Delta$ is defined in Assumption \ref{assumption 1} and $|S|$ and $|Q|$ denotes the size of the support set and query set, respectively.
	\end{theorem}
	The result in Theorem 1 demonstrates that after running CAMEL for $\mathcal{O}(\frac{1+\rho(\sigma+\epsilon)/6}{\epsilon^2} + \frac{B}{\sigma^2} + \frac{(B|Q|+|S|)}{\tilde{\sigma}^2})$ iterations, we are able to find a point $\bm{\theta^\dag}$ at which the expected gradient norm $\mathbb{E}[||\nabla \mathcal{L}_{meta}(\bm{\theta^\dag})||]$ satisfies \eqref{e}.
	
	\section{Experiments}
	We train the model on 3 datasets: RadioML 2016.10A \cite{o2016convolutional}, a dataset with 220,000 total samples, 20,000 samples for each class and 11,000 samples for each SNR, consists of $2\times128$ dimension input X in 11 classes. The 11 classes correspond to 11 modulation types: 8PSK, AM-DSB, AM-SSB, BPSK, CPFSK, GFSK, PAM4, QAM16, QAM64, QPSK, WBFM. And RadioML 2016.04C \cite{o2016convolutional}, a synthetic dataset, is generated with GNU Radio, consisting of about 110 thousand signals. These samples are uniformly distributed in SNR from -20dB to +20dB and tagged so that we can evaluate performance on specific subsets. Actually 2016.10A represents a cleaner and more normalized version of the 2016.04C dataset. The third one is SIGNAL2020.02 \cite{9392373}, whose data is modulated at a rate of 8 samples per symbol, while 128 samples per frame, with 20 different SNRs, even values between [2dB, 40dB].
	
	\subsection{Experimental setup}
	
	The CAMEL is implemented in Pytorch \cite{paszke2019pytorch} with python on a RTX3090 Graphics Processing Units, and trained using the Adam optimizer \cite{kingma2014adam}. In the classification experiments of three datasets, RadioML 2016.04C, RadioML 2016.10A and SIGNAL 2020.02, the default hyper-parameters are as follows: the training epochs are 400,000; the meta batch size is 2; the meta-level outer learning rate is 0.001 and the task-level inner update learning rate is 0.1; the task-level inner update step is 5 and the update step for fine-tuning is 10. All of our experiments use the same hyper-parameter as the default setting. We change the support set shot number in 1 and 5 to have different results of 5-way 1-shot case and 5-way 5-shot case. 
	
	\subsection{Our Model}
	
	First, we study the influence of adding a multi-head self attention mechanism in this network, which can focus attention on important information. We perform a multi-head attention with 8 heads. Instead of performing a single attention function with input $d_x$-dimensional keys, values and queries, it is found beneficial to linearly project the queries, keys and values $h$ times with different, learned linear projections to $d_k$, $d_q$, $d_v$ dimensions, respectively. Then perform the attention function in parallel, concatenate the outputs and do the projection again to get the final result \cite{NEURIPS2019_876e8108}. In our experiments, as illustrated in Table~\ref{table1}, the performance is much better with the addition of the multi-head attention mechanism. As the batch size increases, the performance improves while increasing computation and time-consuming. To make a trade-off, we set the batch size to be 64 when using multi-head attention. We observe that the model with attention mechanism demonstrates a greater ability to increase the accuracy owing to various improvements.
	
	\begin{table}
		\caption{CAMEL: compare with other meta-learning models on datasets RADIOML 2016.10A and SIGNAL 2020.02. The method `MAML+attention' indicates adding multi-head attention mechanism on the origin MAML model, `MAML+complex' means constructing complex-valued neural network in the MAML model and `CT \cite{yang2020complex}' represents the Complex Transformer model using 8 attention functions to represent the complex-valued attention. The $\pm$ shows 95\% confidence intervals over tasks.}
		\label{table1}
		\centering
		\footnotesize{
			\begin{tabular}{lllll}
				\toprule
				&\multicolumn{1}{c}{RADIOML 2016.10A} & \multicolumn{3}{c}{SIGNAL2020.02}   \\
				\cmidrule(r){2-5}
				
				Method & 1-shot & 5-shot & 1-shot & 5-shot\\
				\midrule
				MAML \cite{finn2017model}  & 86.57\% & 94.50\% & 43.26\% & 67.77\% \\
				MAML+attention  & 95.80\% & 97.70\% & 54.44\% & 63.33\% \\
				MAML+complex  & 91.40\% & 96.38\% & 59.50\% & 64.00\%\\
				SNAIL \cite{mishra2018simple}  & 71.18\% & 78.48\% &35.01\% & 36.34\%\\
				Reptilec \cite{nichol2018firstorder}  & 69.16\% & 92.32\% &55.01\% & 69.39\% \\
				MAML+complex+CT \cite{yang2020complex}  & 96.40\% & 97.50\% & 58.40\% & 69.80\%\\
				\textbf{CAMEL (ours)}  & \textbf{97.23\%$\pm$0.13\%} & \textbf{98.22\%$\pm$0.08\%} & \textbf{64.80\%$\pm$0.10\%} & \textbf{74.27\%$\pm$0.15\%} \\
				\bottomrule
		\end{tabular}}
	\end{table}
	
	\begin{table}
		\caption{CAMEL: compare with other meta-learning models in detail on the dataset RADIOML 2016.04C. The $\pm$ shows 95\% confidence intervals over tasks. CAMEL outperforms all other meta-learning models listed.}
		\label{table2}
		\centering
		\begin{tabular}{lll}
			\toprule
			& \multicolumn{1}{c}{RADIOML 2016.04C}  \\
			\cmidrule(r){2-3}
			
			Method & 1-shot & 5-shot\\
			\midrule
			MAML \cite{finn2017model} & 88.93\%$\pm$0.13\% & 93.59\%$\pm$0.62\%  \\
			MAML+attention & 92.12\%$\pm$0.22\% & 95.51\%$\pm$0.05\%  \\
			MAML+complex & 91.65\%$\pm$0.35\% &96.28\%$\pm$0.53\% \\
			SNAIL \cite{mishra2018simple} & 89.21\%$\pm$0.75\% & 96.90\%$\pm$0.19\% \\
			Reptile \cite{nichol2018firstorder} & 87.08\%$\pm$2.88\% & 92.07\%$\pm$5.65\%  \\
			MAML+complex+CT \cite{yang2020complex} & 93.58\%$\pm$1.15\% & 96.52\%$\pm$0.08\% \\
			\textbf{CAMEL (ours)} & \textbf{96.30\%$\pm$0.22\%} & \textbf{97.51\%$\pm$0.15\%} \\
			\bottomrule
		\end{tabular}
	\end{table}
	
	Further study concerns the influence of adding a complex-valued neural network, because we notice that complex numbers could have a richer representational capacity. For these signals inputs, using complex number can probably obtain more useful details than real numbers and could also facilitate noise-robust memory retrieval mechanisms \cite{trabelsi2018deep}. We need to deal with the complex building blocks to construct a complex number neural network: representing of complex numbers, Complex gradient vectors, complex weight initialization, complex convolutions, complex-valued activation,  complex-valued normalization and complex-valued multi-head attention mechanism. These blocks are determined by their own algorithm and the algorithm of complex numbers. We figure out from the results in Table~\ref{table1} and Table~\ref{table2} that this complex features improve the classification accuracy in both 5-way 1-shot and 5-way 5-shot cases with different datasets.
	
	In the training process, we adjust the number of convolution kernels to 128. For the multi-head attention part, we set the source sequence length and output sequence length to 64, number of heads to 8. We observe that such complex-valued models are more competitive than their real valued counterparts. These build our final model: CAMEL, Model-Agnostic Meta-Learning with features of multi-head attention and complex-valued neural network. Compared with the other meta-learning models, CAMEL achieves the best classification accuracy.
	
	The Complex Transformer \cite{yang2020complex} implements complex attention in another way: It rewrites all complex functions into two separate real functions and computes the multiplication of queries, keys and values to get the complex attention with 8 attention functions having different inputs. We also conduct SNAIL \cite{mishra2018simple}, which combines a casual attention operation over the context produced by temporal convolutions, and Reptile \cite{nichol2018firstorder}, which uses only first-order derivatives for meta-learning updates. To have a comparison, Table~\ref{table1} and Table~\ref{table2} list the accuracies of several models based on MAML applied on different datasets. Results in thses two tables demonstrate that our model CAMEL have the state-of-the-art performance among all. In particular, some models are not well performed on the task in the dataset SIGNAL2020.02, but our model CAMEL still has a stable and great performance on this challenging task. Figure~\ref{fig3} indicates that CAMEL could get the highest accuracy at a relatively fast convergence speed. The results also show that, on these challenging signal classification tasks, the CAMEL model apparently outperforms other meta-learning models in accuracy and stability, which could be figured out from the smooth accuracy curves and narrow confidence intervals for CAMEL model in both 1 shot and 5 shot cases.
	
	\begin{figure}[htbp]
		\subfigure 
		{
			\begin{minipage}{6.8cm}
				\centering          
				\includegraphics[scale=0.45]{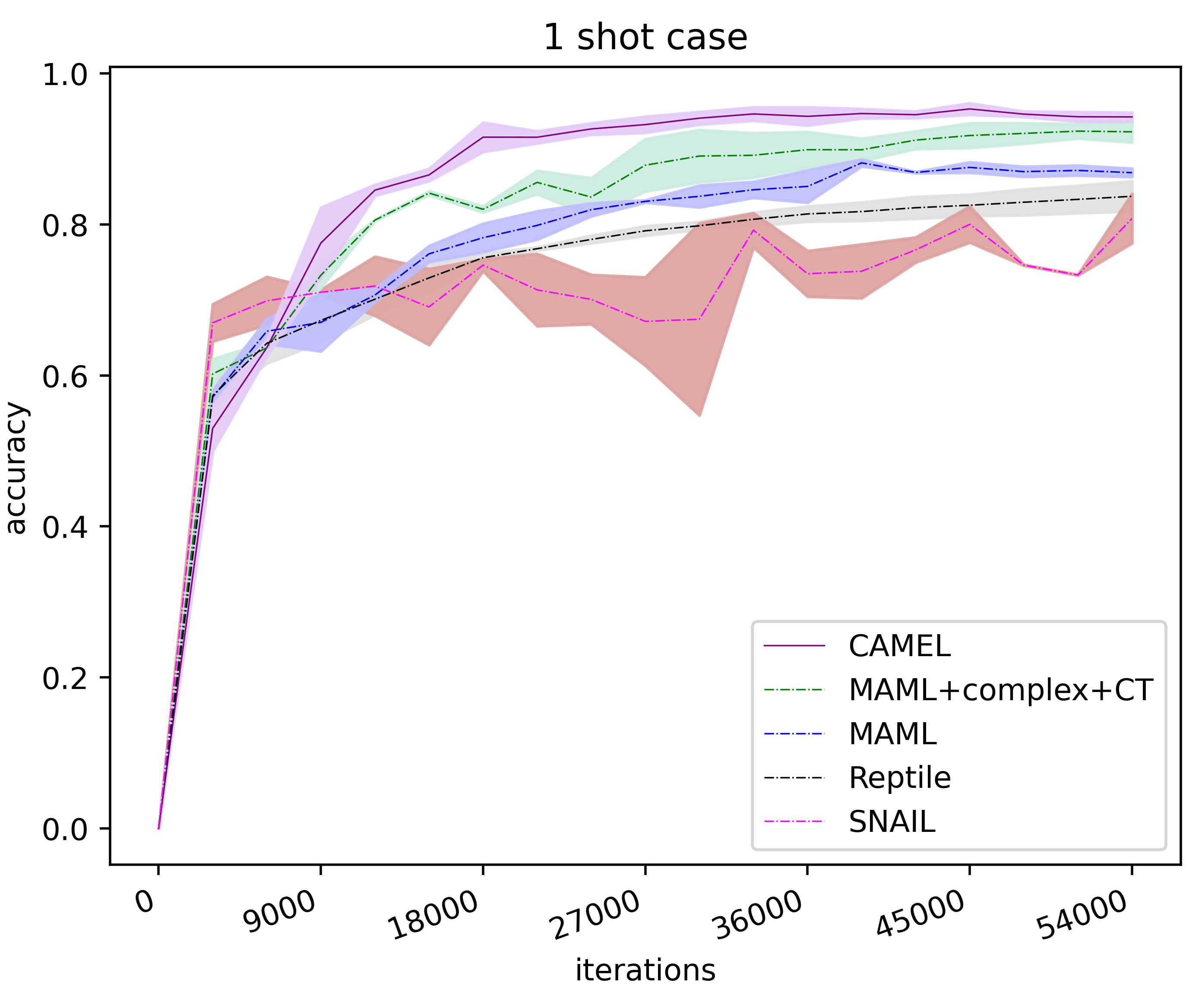}   
			\end{minipage}
		}
		\subfigure 
		{
			\begin{minipage}{6.8cm}
				\centering      
				\includegraphics[scale=0.45]{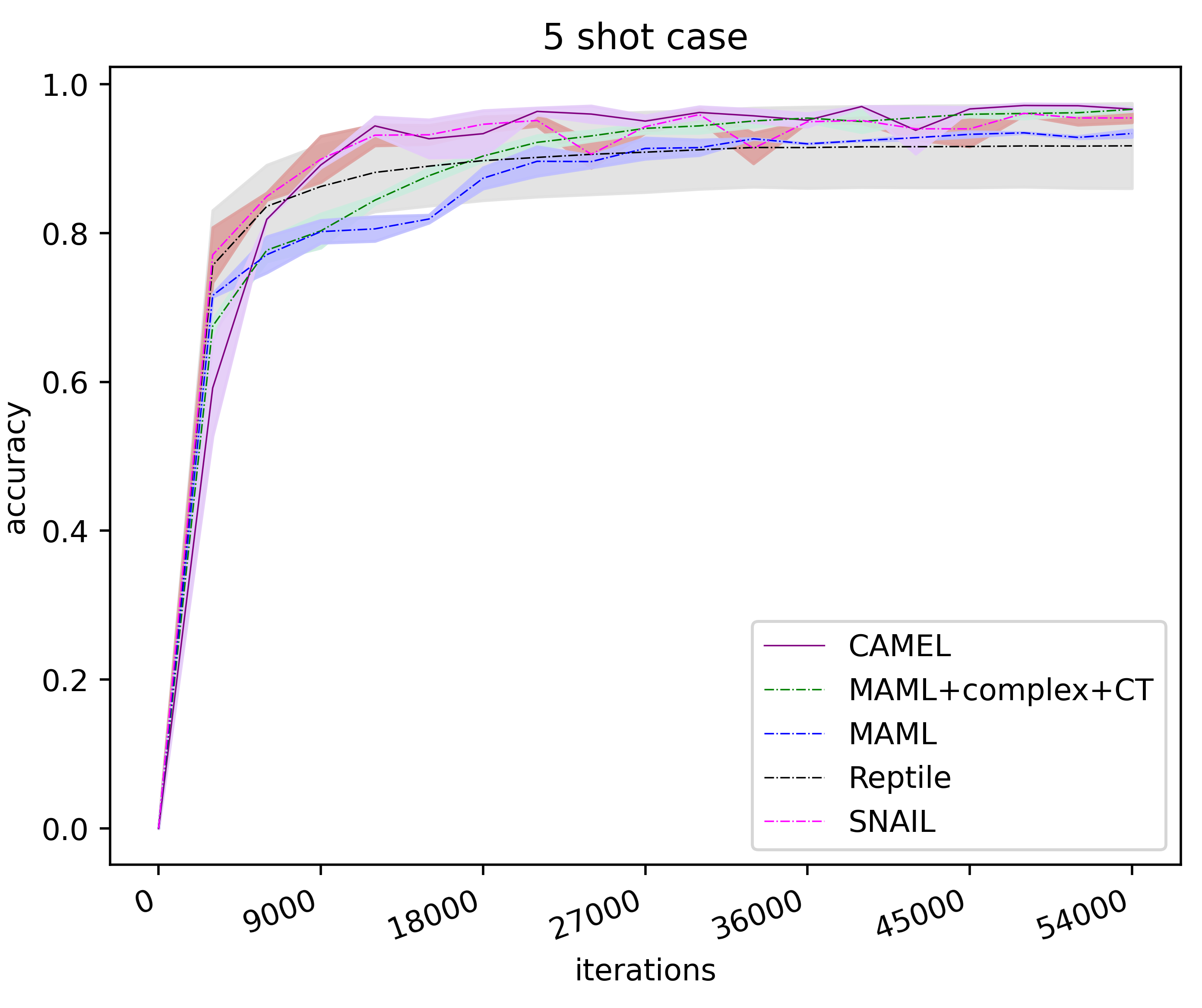}   
			\end{minipage}
		}
		
		\caption{CAMEL: compare the convergence curves of accuracy with other meta-learning models for classification tasks on the dataset RADIOML 2016.04C. This pair of images shows the accuracy curves at 95\% confidence interval over the same classification task. Left: 5-way 1-shot learning. Right: 5-way 5-shot learning. The results indicate that our model CAMEL reaches the highest accuracy at a relatively fast convergence rate. } 
		\label{fig3}  
	\end{figure}
	
	\subsection{Ablation study}
	
	In this section, we have conducted the ablation studies on CAMEL in three scenarios, as shown in Table~\ref{table3}. The first scenario uses samples whose SNR $\geq$ 0, of which 75\% is selected as the training set and 25\% is selected as the test set. For the second scenario, showed in the column "SNR = 0" in Table~\ref{table3}, we pick samples with SNR=0 and randomly select 75\% of them to form the training set and 25\% of them as the test set. The third scenario forms the (Prediction-Other) P-O set as follow: pick 5 classes of signal samples (SNR $\geq$ 0) as set P and the rest 5 classes of samples (SNR $\geq$ 0) form set O. Pick all samples in set O and 5\% of samples in set P as the training set. The remaining 95\% of samples in set P constitute the test set.
	
	On 3 training and testing sets mentioned above, we construct the MAML model first, and then add some features on it step by step. We add attention components and complex numbers separately and together. From the results we observe that in CAMEL, all the features added on the original MAML model help improve the classification accuracy.
	
	\begin{table}
		\caption{Ablation study on CAMEL in three scenarios. For the MAML, add multi-head attention mechanism and use complex numbers step by step. Our model CAMEL combines MAML, complex-valued neural network and complex-valued multi-head attention component. The total experiment consists of 4 models on 3 training and test sets.}
		\label{table3}
		\centering
		\begin{tabular}{lllll}
			\toprule
			Accurancy & SNR $\geq$ 0 & SNR = 0 & P-O set \\
			\midrule
			MAML & 87.20\% & 81.64\% & 89.06\% \\
			MAML+attention & 93.00\% & 87.26\% & 91.90\% \\
			MAML+complex & 91.10\% & 91.75\% &91.30\% \\
			\textbf{CAMEL (ours)} & \textbf{93.70\%} & \textbf{92.10\%} & \textbf{96.30\%} \\
			\bottomrule
		\end{tabular}
	\end{table}
	
	\section{Conclusion}
	In this paper, we have proposed a complex domain attentional meta-learning framework for signal recognition named CAMEL. CAMEL utilizes complex-valued neural networks and attention to provide prior knowledge, i.e., complex domain and temporal information, which helps CAMEL improve performance and prevent overfitting. As two byproducts of CAMEL, we have designed the complex-valued meta-learning and complex-valued attention, which can be of independent interest. With second-order information, CAMEL is able to find first-order stationary points of general nonconvex problems. Furthermore, CAMEL has achieved the state-of-the-art results on extensive datasets. Finally, the ablation studies in three scenarios have demonstrated the effectiveness of the components of CAMEL.
	\newpage
	
	\bibliographystyle{ieeetr}
	\bibliography{paper}
	
	\newpage
	
	\appendix
	\begin{table}[]
		\caption{Definition of mathematical symbols}
		\centering
		\begin{tabular}{c|c}
			\hline
			$\bm{\theta}$ & Meta-parameters of model \\
			$\alpha$ & Learning rate\\
			$\mathcal{T}_i$ & Task $i$ i.e. $\{S_i, Q_i\}$\\
			$p(\mathcal{T})$ & Task probability distribution \\
			$\beta$ & Meta learning rate\\
			$L$ & $\max_i{L_i}$ \\
			$\rho$ & $\max_i{\rho_i}$\\
			$B$ & Batch size\\
			$\sigma^2$ & The variance of gradient $\nabla f_i(\bm{\theta})$\\
			$\tilde{\sigma}^2$ & The variance of gradient $\nabla f_i(\bm{\theta}, d)$ \\
			$d$ & Random minibatch drawn from the dataset of task $i$\\
			$\beta_k$ & The $k$-th step meta learning rate\\
			$L_i$ & Lipschitz constant of $\nabla f_i$ \\
			$\rho_i$ & Lipschitz constant of $\nabla^2 f_i$ \\
			$f_i$ & Loss function for task $i$ \\
			$F_i$ & The loss after the update of gradient descent \\
			$\mathcal{I}$ & The set of all tasks \\
			$\mathbf{x}$ & Complex-valued input vector \\ 
			$\mathbf{b}$ & Complex-valued bias vector\\
			$\mathcal{L}_{S_i}\left(\bm{\theta}\right)$ & The loss on the support set of task $i$\\
			$\mathcal{L}_{meta}(\bm{\theta})$ & Meta-objective\\
			\hline
		\end{tabular}
		\label{tab:symbol}
	\end{table}
	
	
	\section{Proof of Lemma \ref{lemma1}}
	\setcounter{lemma}{0}
	\begin{lemma}
		If a function $g$ is complex analytic, the time complexity of the derivative of $g$ in IQCVNNs are twice that of the complex derivative of $g$ in CDCVNNs. 
	\end{lemma}
	\begin{proof}
		We consider two scenarios, the derivative of simple analytic function and composite analytic function.\\
		1. For a simple analytic function $g(\mathbf{z})$, in CDCVNNs, the complex derivative of $g$ with respect to a complex vector $\mathbf{z}$ is equal to $\dfrac{\partial g(\mathbf{z})}{\partial \mathbf{z}} = \dfrac{\partial \Re(g(\mathbf{z}))}{\partial \Re(\mathbf{z})} + j\dfrac{\partial \Im(g(\mathbf{z}))}{\partial \Re(\mathbf{z})}$. While IQCVNNs considers $\mathbf{z} = \begin{bmatrix} \Re(\mathbf{z})\\\Im(\mathbf{z}) \end{bmatrix}$ and $g(\mathbf{z}) = \begin{bmatrix} \Re(g(\mathbf{z}))\\\Im(g(\mathbf{z})) \end{bmatrix}$, therefore 
		\begin{equation}\label{IQ}
			\begin{aligned}
				\dfrac{\partial g(\mathbf{z})}{\partial \mathbf{z}} &= \dfrac{\partial}{\partial \begin{bmatrix} \Re(\mathbf{z})\\\Im(\mathbf{z}) \end{bmatrix}}   \begin{bmatrix} \Re(g(\mathbf{z}))\\\Im(g(\mathbf{z})) \end{bmatrix}^T \\
				&= \begin{bmatrix} \dfrac{\partial \Re(g(\mathbf{z}))}{\partial \Re(\mathbf{z})} & \dfrac{\partial \Im(g(\mathbf{z}))}{\partial \Re(\mathbf{z})}\\\dfrac{\partial \Re(g(\mathbf{z}))}{\partial \Im(\mathbf{z})} & \dfrac{\partial \Im(g(\mathbf{z}))}{\partial \Im(\mathbf{z})} \end{bmatrix}.
			\end{aligned}
		\end{equation}
		Thus, in this scenario, the time complexity of the derivative of $g$ in IQCVNNs are twice that of the complex derivative of $g$ in CDCVNNs. \\
		2. For a composite analytic function $g(h(\mathbf{z}))$ where $\mathbf{u} = h(\mathbf{z})$ is also complex analytic, in CDCVNNs, the complex derivative of $g$ with respect to a complex vector $\mathbf{z} \in \mathbb{C}^m$ can be computed according to the complex chain rule.
		\begin{equation} \label{25}
			\dfrac{\partial g(\mathbf{u})}{\partial \mathbf{z}} =  \dfrac{\partial g(\mathbf{u})}{\partial \mathbf{u}} \dfrac{\partial \mathbf{u}}{\partial \mathbf{z}} + \dfrac{\partial g(\mathbf{u})}{\partial \mathbf{u}^*} \dfrac{\partial \mathbf{u}^*}{\partial \mathbf{z}}
		\end{equation}
		Owing to the fact that $\mathbf{u} = h(\mathbf{z})$ is complex analytic, $\dfrac{\partial \mathbf{u}^*}{\partial \mathbf{z}}$ is equal to zero. So, \eqref{25} can be simplified to 
		\begin{equation} \label{26}
			\dfrac{\partial g(\mathbf{u})}{\partial \mathbf{z}} =  \dfrac{\partial g(\mathbf{u})}{\partial \mathbf{u}} \dfrac{\partial \mathbf{u}}{\partial \mathbf{z}}
		\end{equation}
		where $\dfrac{\partial g(\mathbf{u})}{\partial \mathbf{u}}$ and $\dfrac{\partial \mathbf{u}}{\partial \mathbf{z}} \in \mathbb{C}^m$. Hence, the time complexity of the complex derivative of $g$ in CDCVNNs is $\mathcal{O}(2*2*m) = \mathcal{O}(4m)$. However, in IQCVNNs, we have
		\begin{equation}\label{tensor}
			\begin{aligned}
				\dfrac{\partial g(\mathbf{u})}{\partial \mathbf{z}} &= \dfrac{\partial}{\partial \begin{bmatrix} \Re(\mathbf{u})\\\Im(\mathbf{u}) \end{bmatrix}} \begin{bmatrix} \Re(g(\mathbf{u}))\\\Im(g(\mathbf{u})) \end{bmatrix}^T \dfrac{\partial}{\partial \begin{bmatrix} \Re(\mathbf{z})\\\Im(\mathbf{z}) \end{bmatrix}}  \begin{bmatrix} \Re(\mathbf{u})\\\Im(\mathbf{u}) \end{bmatrix}^T \\
				&= \begin{bmatrix} \dfrac{\partial \Re(g(\mathbf{u}))}{\partial \Re(\mathbf{u})} & \dfrac{\partial \Im(g(\mathbf{u}))}{\partial \Re(\mathbf{u})}\\\dfrac{\partial \Re(g(\mathbf{u}))}{\partial \Im(\mathbf{u})} & \dfrac{\partial \Im(g(\mathbf{u}))}{\partial \Im(\mathbf{u})} \end{bmatrix} \begin{bmatrix} \dfrac{\partial \Re(\mathbf{u})}{\partial \Re(\mathbf{z})} & \dfrac{\partial \Im(\mathbf{u})}{\partial \Re(\mathbf{z})}\\\dfrac{\partial \Re(\mathbf{u})}{\partial \Im(\mathbf{z})} & \dfrac{\partial \Im(\mathbf{u})}{\partial \Im(\mathbf{z})} \end{bmatrix}
			\end{aligned}
		\end{equation}
		where the size of each above tensor in \eqref{tensor} is $(2,2,m)$. Hence, the time complexity of the derivative of $g$ in IQCVNNs is $\mathcal{O}(2*2*2*m) = \mathcal{O}(8m)$.
		Note that the composite function of $N$ layers, which can be seen as $N-1$ composite functions of two layers calculated serially. As a result, in CDCVNNs, the time complexity of the complex derivative of $g$ is $\mathcal{O}(4m(N-1))$, while in IQCVNNs, the time complexity of the derivative of $g$ is $\mathcal{O}(8m(N-1))$. Hence, the Lemma holds in the scenario of the derivative of composite analytic function. \\
		To sum up, the Lemma \ref{lemma1} is established in both two scenarios. This completes the proof.
	\end{proof}
	
	\section{Proof of Lemma \ref{lemma2}}
	\begin{lemma} 
		The complex-valued convolutional layer and complex-valued fully connected layer is complex analytic.
	\end{lemma}
	\begin{proof}
		It is obviously that the complex-valued convolution layer and complex-valued fully connected layer are linear and continuous. Assume that a linear function $g(\mathbf{x}) = \mathbf{A}^T \mathbf{x} + \mathbf{b}$ is continuous with respect to a complex vector $\mathbf{x}$, then we can obtain
		\begin{equation*}
			\begin{aligned}
				&\dfrac{\partial \Re (g(\mathbf{x}))}{\partial \Re(\mathbf{x})} =  \dfrac{\partial \Re(\mathbf{A})^T \Re(\mathbf{x})-\Im(\mathbf{A})^T \Im(\mathbf{x})+\Re(\mathbf{b})}{\partial \Re(\mathbf{x})}= \Re(\mathbf{A}),\\
				&\dfrac{\partial \Im(g(\mathbf{x}))}{\partial \Im(\mathbf{x})} = \dfrac{\partial \Im(\mathbf{A})^T \Re(\mathbf{x})+\Re(\mathbf{A})^T \Im(\mathbf{x}) +\Im(\mathbf{b})}{\partial \Im(\mathbf{x})}= \Re(\mathbf{A}).
			\end{aligned}
		\end{equation*}
		In a similar way,
		\begin{equation*}
			\dfrac{\partial \Re(g(\mathbf{x}))}{\partial \Im(\mathbf{x})} = -\Im(\mathbf{A}), \quad \dfrac{\partial \Im(g(\mathbf{x}))}{\partial \Re(\mathbf{x})} = \Im(\mathbf{A}).
		\end{equation*}
		Therefore, 
		\begin{equation*}
			\left \{
			\begin{aligned}
				&\dfrac{\partial \Re (g(\mathbf{x}))}{\partial \Re(\mathbf{x})} = \dfrac{\partial \Im(g(\mathbf{x}))}{\partial \Im(\mathbf{x})}\\
				&\dfrac{\partial \Re(g(\mathbf{x}))}{\partial \Im(\mathbf{x})} = -\dfrac{\partial \Im(g(\mathbf{x}))}{\partial \Re(\mathbf{x})}.
			\end{aligned}
			\right.
		\end{equation*}
		According to the function $g(\mathbf{x}) $ is continuous and satisfies the Cauchy-Riemann equations, the linear function is complex analytic. Hence, the complex-valued convolution layer and complex-valued fully connected layer are complex analytic. This completes the proof.
	\end{proof}
	
	\section{Proof of Lemma \ref{cr}}
	\begin{lemma}
		$\forall g: \mathbb{C} \rightarrow \mathbb{R}$, $g$ is analytic if and only if $g$ is a constant function.
	\end{lemma}
	\begin{proof}
		Assume a function $g \in \{\mathbb{C} \rightarrow \mathbb{R}\}$ is analytic. Then $g$ has to satisfy the Cauchy-Riemann equations:
		\begin{equation*}
			\left \{
			\begin{aligned}
				&\dfrac{\partial \Re (g(\mathbf{x}))}{\partial \Re(\mathbf{x})} = \dfrac{\partial \Im(g(\mathbf{x}))}{\partial \Im(\mathbf{x})} =  \dfrac{\partial \ 0}{\partial \Im(\mathbf{x})}=0\\
				&\dfrac{\partial \Re(g(\mathbf{x}))}{\partial \Im(\mathbf{x})} = -\dfrac{\partial \Im(g(\mathbf{x}))}{\partial \Re(\mathbf{x})} = - \dfrac{\partial \ 0}{\partial \Re(\mathbf{x})}= 0 
			\end{aligned}
			\right.
		\end{equation*}
		where $\mathbf{x}$ is the complex input vector. Since the partial derivatives of $g$ are all equal to 0, $g$ is a constant function. This completes the proof.
	\end{proof}
	
	\section{Definition Recall} \label{Definition Recall}
	In this section, we recall the definitions of complex derivative, analytic function, and the Cauchy-Riemann equations. 
	
	\textbf{Complex derivative} \ Let $g(z) \in \mathbb{C}$, where $z \in \mathbb{C}$. If $g(z)$ is continuous at a point $z_0$, we can define its complex derivative as:
	\begin{equation}
		g'(z_0) = \left.\dfrac{dg}{dz}\right|_{z=z_0} = \lim_{\Delta z \to 0} \dfrac{g(z_0+\Delta z)-g(z_0)}{\Delta z}.
	\end{equation}
	This is similar to the definition of the derivative for the function of a real variable. In the real case, the existence of the derivative implies that the limits of $\dfrac{dg}{dz}$ both exist and are equal when the point $z_0+\Delta z$ converges to $z_0$ from both the left $(\Delta < 0)$ and right $(\Delta > 0)$ directions. However, in the complex case, it means that the limits of $\dfrac{dg}{dz}$ exist and are equal when the point $z_0+\Delta z$ converges to $z_0$ from any directions in the complex plane.  If a function satisfies this property at a point $z_0$, we say that the function is \emph{complex-differentiable} at $z_0$.
	
	\textbf{Analytic function} \ If a function $g(z)$ is complex-differentiable for all points $z$ in some domain $D \subset \mathbb{C}$ , then $g(z)$ is said to be analytic, i.e., $g(z)$ is a \emph{complex analytic function} also known as holomorphic function, in $D$. 
	
	\textbf{The Cauchy-Riemann equations} \ The Cauchy-Riemann equations are a pair of real partial differential equations, and their complex analytic function needs to satisfy:
	\begin{equation}
		\left \{
		\begin{aligned}
			&\dfrac{\partial \Re (g(z))}{\partial \Re(z)} = \dfrac{\partial \Im(g(z))}{\partial \Im(z)}\\
			&\dfrac{\partial \Re(g(z))}{\partial \Im(z)} = -\dfrac{\partial \Im(g(z))}{\partial \Re(z)} 
		\end{aligned}
		\right.
	\end{equation}
	where $\Re(\cdot)$ and $\Im(\cdot)$ denote the real and imaginary parts of the complex number, respectively. The necessary and sufficient condition for $g(z)$ to be complex analytic function in $D$ is that the function $g(z)$ is continuous and satisfies the Cauchy-Riemann equations in $D$.
	
	\section{Proof of Lemma \ref{lemma4}}
	\begin{lemma}
		In response to complex meta-parameters $\bm{\theta}$, we have
		\begin{equation}
			\nabla_{\bm{\theta}} \mathcal{L}_{meta}(\bm{\theta}) = \dfrac{1}{B} \Sigma_{\{S_i, Q_i\} \sim p(\mathcal{T})}  \left( \mathbf{I} - \alpha \mathbf{H}^i_{\bm{\theta} \bm{\theta}} \right) \nabla_{\bm{\theta}'_i} \mathcal{L}_{Q_i}(\bm{\theta}'_i) - \alpha  \mathbf{H}^i_{\bm{\theta}^* \bm{\theta}} \nabla_{\bm{(\bm{\theta}'_i)^*}} \mathcal{L}_{Q_i}(\bm{\theta}'_i). \tag {11}
		\end{equation}
	\end{lemma}
	\begin{proof}
		According to \eqref{7}, it is obviously that 
		\begin{equation}\label{31}
			\nabla_{\bm{\theta}} \mathcal{L}_{meta}(\bm{\theta}) = \dfrac{1}{B} \Sigma_{\{S_i, Q_i\} \sim p(\mathcal{T})} \nabla_{\bm{\theta}} \mathcal{L}_{Q_i}(\bm{\theta}'_i).
		\end{equation}
		Note that, since $ \mathcal{L}_{Q_i}: \mathbb{C} \rightarrow \mathbb{R}$, following the definition of complex gradient vector \eqref{cgv}, we have 
		\begin{equation}\label{32}
			\begin{aligned}
				\nabla_{\bm{\theta}} \mathcal{L}_{Q_i}(\bm{\theta}'_i) &= 2\dfrac{\partial \mathcal{L}_{Q_i}(\bm{\theta}'_i)}{\partial \bm{\theta}^*}\\
				&= 2\left(\dfrac{\partial \mathcal{L}_{Q_i}(\bm{\theta}'_i)}{\partial \bm{\theta}}\right)^* \\
				&= 2\left(\dfrac{\partial \mathcal{L}_{Q_i}(\bm{\theta}'_i)}{\partial \bm{\theta}'_i} \dfrac{\partial \bm{\theta}'_i}{\partial \bm{\theta}} + \dfrac{\partial \mathcal{L}_{Q_i}(\bm{\theta}'_i)}{\partial \left(\bm{\theta}'_i\right)^*} \dfrac{\partial \left(\bm{\theta}'_i\right)^*}{\partial \bm{\theta}}\right)^*\\
				&= 2\left(\dfrac{\partial \mathcal{L}_{Q_i}(\bm{\theta}'_i)}{\partial (\bm{\theta}'_i)^*} \left(\dfrac{\partial \bm{\theta}'_i}{\partial \bm{\theta}}\right)^* + \dfrac{\partial \mathcal{L}_{Q_i}(\bm{\theta}'_i)}{\partial \left(\bm{\theta}'_i\right)^{**}} \left(\dfrac{\partial \left(\bm{\theta}'_i\right)^*}{\partial \bm{\theta}}\right)^*\right)\\
				&= \nabla_{\bm{\theta}'_i} \mathcal{L}_{Q_i}(\bm{\theta}'_i) \left(\dfrac{\partial \bm{\theta}'_i}{\partial \bm{\theta}}\right)^* + \nabla_{\bm{(\bm{\theta}'_i)^*}} \mathcal{L}_{Q_i}(\bm{\theta}'_i) \left(\dfrac{\partial \left(\bm{\theta}'_i\right)^*}{\partial \bm{\theta}}\right)^*
			\end{aligned}
		\end{equation}
		where the second equality is because the output of $\mathcal{L}_{Q_i}$ is real, the third equality follows the complex chain rule, and the last equality is given by the definition of complex gradient vector. Next, according to inner-loop update process \eqref{g}, we have
		\begin{equation}
			\begin{aligned}
				\left(\dfrac{\partial \bm{\theta}'_i}{\partial \bm{\theta}}\right)^* &= \left(\dfrac{\partial \left(\bm{\theta} - \alpha \nabla_{\bm{\theta}} \mathcal{L}_{S_i}\left(\bm{\theta}\right)\right)}{\partial \bm{\theta}}\right)^*\\
				&= \left(\mathbf{I} - \alpha \dfrac{\partial \nabla_{\bm{\theta}} \mathcal{L}_{S_i}\left(\bm{\theta}\right)}{\partial \bm{\theta}}\right)^*\\
				&= \mathbf{I} - \alpha \left(\dfrac{\partial \nabla_{\bm{\theta}} \mathcal{L}_{S_i}\left(\bm{\theta}\right)}{\partial \bm{\theta}}\right)^* .\\
			\end{aligned}
		\end{equation}
		Similarly,
		\begin{equation}
			\begin{aligned}
				\left(\dfrac{\partial \left(\bm{\theta}'_i\right)^*}{\partial \bm{\theta}}\right)^* &= \left(\dfrac{\partial \left(\bm{\theta} - \alpha \nabla_{\bm{\theta}} \mathcal{L}_{S_i}\left(\bm{\theta}\right)\right)^*}{\partial \bm{\theta}}\right)^*\\
				&= \left(\dfrac{\partial \bm{\theta}^* - \alpha \nabla_{\bm{\theta}^*} \mathcal{L}_{S_i}\left(\bm{\theta}\right)}{\partial \bm{\theta}}\right)^*\\
				&= - \alpha \left(\dfrac{\partial \nabla_{\bm{\theta}^*} \mathcal{L}_{S_i}\left(\bm{\theta}\right)}{\partial \bm{\theta}}\right)^* .\\
			\end{aligned}
		\end{equation}
		Now, using \eqref{H}, we can write \eqref{32} as
		\begin{equation} \label{35}
			\nabla_{\bm{\theta}} \mathcal{L}_{Q_i}(\bm{\theta}'_i) = \left( \mathbf{I} - \alpha \mathbf{H}^i_{\bm{\theta} \bm{\theta}} \right) \nabla_{\bm{\theta}'_i} \mathcal{L}_{Q_i}(\bm{\theta}'_i) - \alpha  \mathbf{H}^i_{\bm{\theta}^* \bm{\theta}} \nabla_{\bm{(\bm{\theta}'_i)^*}} \mathcal{L}_{Q_i}(\bm{\theta}'_i).
		\end{equation}
		Plugging \eqref{35} in \eqref{31} yields
		\begin{equation}
			\nabla_{\bm{\theta}} \mathcal{L}_{meta}(\bm{\theta}) = \dfrac{1}{B} \Sigma_{\{S_i, Q_i\} \sim p(\mathcal{T})}  \left( \mathbf{I} - \alpha \mathbf{H}^i_{\bm{\theta} \bm{\theta}} \right) \nabla_{\bm{\theta}'_i} \mathcal{L}_{Q_i}(\bm{\theta}'_i) - \alpha  \mathbf{H}^i_{\bm{\theta}^* \bm{\theta}} \nabla_{\bm{(\bm{\theta}'_i)^*}} \mathcal{L}_{Q_i}(\bm{\theta}'_i).
		\end{equation}
		This completes the proof.
	\end{proof}

	\section{Complex-valued Neural Networks}
	\label{CVNN}
	Neural networks require back-propagation to update their parameters via first-order derivatives, as do complex-valued neural networks. We would prefer the functions in complex-valued neural networks to be analytic. \emph{Define $\mathbf{x}$ and $\mathbf{b}$ as the complex input vector and complex bias vector for each function, respectively.}
	
	\textbf{Complex-valued convolutional layer} \ The complex-valued convolutional layer implements the convolution operation on complex input signals. Define $\mathbf{A}$ as the complex convolution kernel. Given $\mathbf{x}$, $\mathbf{A}$, and $\mathbf{b}$, since the complex-valued convolutional layer is linear, we are able to compute the real and imaginary parts of its outputs separately as the following
	\begin{subequations}
		\begin{align}\label{Rconv}
			\Re(\textrm{C}_{conv}(\mathbf{x},\mathbf{A},\mathbf{b}) = \Re(\mathbf{A}) \otimes \Re(\mathbf{x}) - \Im(\mathbf{A}) \otimes \Im(\mathbf{x}) + \Re(\mathbf{b}) ,
			\\
			\Im(\textrm{C}_{conv}(\mathbf{x},\mathbf{A},\mathbf{b}) = \Re(\mathbf{A}) \otimes \Im(\mathbf{x}) + \Im(\mathbf{A}) \otimes \Re(\mathbf{x}) + \Im(\mathbf{b}).\label{Iconv}
		\end{align}
	\end{subequations}
	Then, according to the \eqref{Rconv} and \eqref{Iconv}, the complex-valued convolutional layer can be represented as follows:
	\begin{equation*}
		\begin{aligned}
			&\textrm{C}_{conv}(\mathbf{x},\mathbf{A},\mathbf{b}) = \Re(\textrm{C}_{conv}(\mathbf{x},\mathbf{A},\mathbf{b}) + j \Im(\textrm{C}_{conv}(\mathbf{x},\mathbf{A},\mathbf{b})\\
			= & ( \Re(\mathbf{A}) \otimes \Re(\mathbf{x}) - \Im(\mathbf{A}) \otimes \Im(\mathbf{x}) + \Re(\mathbf{b}))  + j (\Re(\mathbf{A}) \otimes \Im(\mathbf{x}) + \Im(\mathbf{A}) \otimes \Re(\mathbf{x}) + \Im(\mathbf{b}))
		\end{aligned}
	\end{equation*}
	where $\otimes$ denotes the convolution operation in real case. 
	
	\textbf{Complex-valued fully connected layer} \ 
	The complex-valued fully connected layer achieves the linear transformation of complex inputs. Define $\mathbf{W}$ as the complex weight matrix. Given $\mathbf{x}$, $\mathbf{W}$, and $\mathbf{b}$, the real and imaginary parts of the outputs of complex-valued fully connected layer can be computed as: 
	\begin{subequations}
		\begin{align}
			\Re(\textrm{C}_{fc}(\mathbf{x},\mathbf{W},\mathbf{b})) = \Re(\mathbf{W})^T\Re(\mathbf{x}) - \Im(\mathbf{W})^T\Im(\mathbf{x}) + \Re(\mathbf{b}),\\
			\Im(\textrm{C}_{fc}(\mathbf{x},\mathbf{W},\mathbf{b})) = \Im(\mathbf{W})^T \Re(\mathbf{x})+\Re(\mathbf{W})^T \Im(\mathbf{x}) +\Im(\mathbf{b}).
		\end{align}
	\end{subequations}
	Similarly, the complex-valued fully connected layer can be expressed as:
	\begin{equation}
		\begin{aligned}
			\textrm{C}_{fc}(\mathbf{x},\mathbf{W},\mathbf{b}) & = \Re(\textrm{C}_{fc}(\mathbf{x},\mathbf{W},\mathbf{b})) + j\Im(\textrm{C}_{fc}(\mathbf{x},\mathbf{W},\mathbf{b}))\\
			& = \Re(\mathbf{W})^T \Re(\mathbf{x})-\Im(\mathbf{W})^T \Im(\mathbf{x})+\Re(\mathbf{b})\\
			& \quad + j  \Im(\mathbf{W})^T \Re(\mathbf{x})+\Re(\mathbf{W})^T \Im(\mathbf{x}) +\Im(\mathbf{b}).
		\end{aligned}
	\end{equation}
	
	\section{Related work}
	\label{related work}
	Recently, meta-learning has demonstrated promising performance in many fields. Khodadadeh et al. \cite{NEURIPS2019_fd0a5a5e} proposed an unsupervised algorithm for model-independent meta-learning for classification tasks. The work \cite{NEURIPS2019_92262bf9} proposed a new method that automatically learns appropriate labels for auxiliary tasks. The work \cite{NEURIPS2019_6fe43269} proposed a new meta-learning method to learn heterogeneous point process models from short event sequence data and relational networks. In addition, the work \cite{NEURIPS2019_7a9a322c} proposed a Dirichlet process mixture for hierarchical Bayesian models with the parameters of arbitrary parametric models. Khodak et al. \cite{NEURIPS2019_f4aa0dd9} built a theoretical framework for the design and understanding of practical meta-learning methods. The authors in \cite{NEURIPS2019_8c235f89} proposed a meta-learning method based on minibatch proximal update for learning effective hypothesis transfer. 
	
	Moreover, the work \cite{NEURIPS2019_072b030b} proposed an implicit MAML algorithm which relies only on the solution to the inner level optimization. The work \cite{NEURIPS2020_1e04b969} a meta-learning approach that avoids the need for this often sub-optimal hand-selection.  The work \cite{NEURIPS2020_4b86ca48} proposed an online structured meta-learning framework. Additionally, the authors in \cite{NEURIPS2020_ee89223a} proposed a new weight update rule that greatly enhances the fast adaptation process. The work \cite{NEURIPS2020_cc3f5463} proposed a meta-learning approach via online changepoint analysis to augment with a differentiable Bayesian changepoint detection scheme. The work \cite{NEURIPS2020_cfee3986} proposed an adversarial querying algorithm for generating adversarially robust meta-learners and thoroughly investigated the causes for adversarial vulnerability.

	\setcounter{theorem}{0}
	\section{Convergence Analysis} \label{Proof of Theorem 1}
	For ease of writing and derivation, in our notation, $f_i(\bm{\theta}) = \mathcal{L}_{S_i}\left(\bm{\theta}\right)$ represents the loss function on the task $i$, $F_i(\bm{\theta})=f_i(\bm{\theta} - \alpha\nabla f_i(\bm{\theta})) = f_i(\bm{\theta}'_i)= \mathcal{L}_{Q_i}(\bm{\theta}'_i)$ represents the loss on the task $i$ after the inner-loop update process, and $F(\bm{\theta}) = \mathcal{L}_{meta}(\bm{\theta})$ represents the meta-objective. By drawing task $i$ from task probability distribution $p(\mathcal{T}_i)$, our optimization problem can be rewritten as 
	\begin{equation} \label{of}
		\min_{\bm{\theta}} F(\bm{\theta})=\mathbb{E}_{i\sim p}[F_i(\bm{\theta})]=\mathbb{E}_{i\sim p}[f_i(\bm{\theta} - \alpha\nabla f_i(\bm{\theta}))].
	\end{equation}
	\begin{definition}
		A random vector $\bm{\theta}_\epsilon \in \mathbb{C}^m$ is called an $\epsilon$-approximate first order stationary point for problem \ref{of} if it satisfies $\mathbb{E}[||\nabla F(\bm{\theta}_\epsilon)||] \leq \epsilon$.
	\end{definition}
	
	Then, we formally state our assumptions as below. 
	\begin{assumption}\label{assumption 1}
		$F$ is bounded below, $\min{F(\bm{\theta})} > -\infty$ and $\Delta \triangleq (F(\bm{\theta}_0) - \min_{\bm{\theta}\in \mathbb{C}}{F(\bm{\theta})})$ is bounded.
	\end{assumption} 
	\begin{assumption}
		Suppose $\mathcal{I}$ denotes the set of all tasks. $\forall i\in\mathcal{I}$, $f_i$ is twice continuously differentiable with respect to $\mathbf{z}$ and $\mathbf{z}^*$ (the second-order Wirtinger derivatives \cite{hjorungnes2011complex} of $f_i$ exist and are continuous) and $L_i$-smooth \cite{zhang2015complex,hirose2012complex}, i.e., 
		\begin{equation}
			\forall\bm{\theta}, \bm{\mu} \in\mathbb{C}^m, ||\nabla f_i(\bm{\theta}) - \nabla f_i(\bm{\bm{\mu}})|| \leq L_i||\bm{\theta} - \bm{\mu}||
		\end{equation}
		where norm $||\mathbf{z}||$ denotes $\sqrt{\mathbf{z}_1\mathbf{z}^*_1 + ... + \mathbf{z}_m\mathbf{z}^*_m}$.
	\end{assumption}
	\begin{assumption}
		$\forall i\in\mathcal{I}$, the Hessian $\nabla^2f_i=\frac{\partial^2 f_i}{\partial \mathbf{v}^* \partial \mathbf{v}^\top}$
		is $\rho_i$-Lipschitz continuous where $\mathbf{v}=(\mathbf{z}, \mathbf{z}^*)^\top\in\mathbb{C}^{2m}$ [3,5], i.e., 
		\begin{equation}
			\forall\bm{\theta}, \bm{\mu} \in\mathbb{C}^m, ||\nabla^2 f_i(\bm{\theta}) - \nabla^2 f_i(\bm{\mu})|| \leq \rho_i||\bm{\theta} - \bm{\mu}||.
		\end{equation}
		Note that we have one $L_i$ and $\rho_i$ for each $f_i$, so in this paper we will use $L=\max_i{L_i}$ and $\rho=\max_i{\rho_i}$ to represent the Lipschitz constant of the gradients and Hessians for all $i\in\mathcal{I}$. We follow the definition of Lipschitz continuous in the complex domain from \cite{zhang2014convergence,zhang2015complex}. 
	\end{assumption}
	\begin{assumption}
		The variance of gradient $\nabla f_i(\bm{\theta})$ is bounded, i.e., for some real-valued parameter $\sigma > 0$, we have $\mathbb{E}_{i\sim p}[||\nabla f(\bm{\theta}) - \nabla f_i(\bm{\theta})||^2] \leq \sigma^2$.
	\end{assumption}
	\begin{assumption} \label{assumption 5}
		Suppose $d\sim\mathcal{T}_i$ denotes a random minibatch drawn from the dataset of task $i$. Then $\forall i\in\mathcal{I}$ and $\forall\bm{\theta}, \bm{\mu} \in\mathbb{C}^m$, the stochastic gradients $\nabla f_i(\bm{\theta}, d)$ and Hessians $\nabla^2 f_i(\bm{\theta}, d)$ have bounded variance, i.e., 
		\begin{subequations}
			\begin{align}
				\mathbb{E}_{d\sim\mathcal{T}_i}[||\nabla f_i(\bm{\theta}, d) - \nabla f_i(\bm{\theta})||^2] \leq \tilde{\sigma}^2,\\
				\mathbb{E}_{d\sim\mathcal{T}_i}[||\nabla^2 f_i(\bm{\theta}, d) - \nabla^2 f_i(\bm{\theta})||^2] \leq \sigma_H^2.
			\end{align}
		\end{subequations}
		where $\tilde{\sigma}$ and $\sigma_H$ are non-negative real-valued constants.
	\end{assumption}
	\begin{lemma} \cite{fallah2020convergence}\label{5}
		Suppose that Assumptions 2-5 hold and $\alpha\in[0,\frac{1}{L}]$. Then consider the definition
		$$\tilde{\beta}(\bm{\theta})=\frac{1}{4L+2\rho\alpha\sum_{i\in\mathcal{B}'}||\tilde{\nabla}f_i(\bm{\theta}, \mathcal{D}_{\beta}^i)||/B'}
		$$
		where $\mathcal{B}'$ is a batch of tasks with size $B'$ which are independently drawn with probability distribution $p(\mathcal{T}_i)$, and for $i\in\mathcal{B}'$, $\mathcal{D}_{\beta}^i$ is a dataset corresponding to task $i$ with size $D_\beta$. If
		$$B' \geq \lceil 0.5(\rho\alpha\sigma/L)^2 \rceil,\ \ D_\beta \geq \lceil 2(\rho\alpha\tilde{\sigma}/L)^2 \rceil
		$$
		are satisfied, then
		$$\mathbb{E}[\tilde{\beta}(\bm{\theta})]\geq \frac{0.8}{L(\bm{\theta})},\ \ \mathbb{E}[\tilde{\beta}(\bm{\theta})^2]\geq \frac{3.125}{L(\bm{\theta}^2)}
		$$
		where $L(\bm{\theta}) \triangleq 4L+2\rho\alpha\mathbb{E}_{i\sim p}||\nabla f_i(\bm{\theta})||$.
	\end{lemma} 
	
	\begin{theorem} \label{theorem 1}
		Suppose that Assumptions 1-5 hold and $\alpha\in(0, \frac{1}{6L}]$. Consider running complex-valued MAML with batch sizes $B\geq 20$ and $|Q|\geq \lceil 2\alpha^2\sigma_H^2 \rceil$. Following the definition in Lemma \ref{5}, let $\beta_k=\tilde{\beta}(\bm{\theta}_k)/12$. Then for any $\epsilon > 0$, complex-valued MAML finds a solution that 
		\begin{equation} 
			\mathbb{E}[||\nabla \mathcal{L}_{meta}(\bm{\theta}_\epsilon)||] \leq \max{\left\{ \sqrt{61(1+\frac{\rho\alpha}{L}\sigma)(\frac{\sigma^2}{B} + \frac{\tilde{\sigma}^2}{B|Q|} + \frac{\tilde{\sigma}^2}{|S|})}, \frac{61\rho\alpha}{L}(\frac{\sigma^2}{B} + \frac{\tilde{\sigma}^2}{B|Q|} + \frac{\tilde{\sigma}^2}{|S|}), \epsilon \right\}}
		\end{equation}
		after running for
		\begin{equation}
			\mathcal{O}(1)\Delta\min\left\{ \frac{L+\rho\alpha(\sigma+\epsilon)}{\epsilon^2}, \frac{LB}{\sigma^2} + \frac{L(B|Q|+|S|)}{\tilde{\sigma}^2} \right\}
		\end{equation}
		iterations, where $\Delta$ is defined in Assumption 1 and $|S|$ and $|Q|$ denotes the size of the support set and query set, respectively.
	\end{theorem}
	\begin{proof}
		Define $G_i(\bm{\theta}) \triangleq \nabla_{\bm{\theta}} F_i(\bm{\theta})$. $G_i(\bm{\theta})$ can be written as
		\begin{equation}
			\begin{aligned}
				G_i(\bm{\theta}) &= \left( \mathbf{I} - \alpha \mathbf{H}^i_{\bm{\theta} \bm{\theta}} \right) \nabla_{\bm{\theta}'_i} f_i(\bm{\theta}'_i) - \alpha  \mathbf{H}^i_{\bm{\theta}^* \bm{\theta}} \nabla_{\bm{(\bm{\theta}'_i)^*}} f_i(\bm{\theta}'_i)\\
				&=  \nabla_{\bm{\theta}'_i} f_i(\bm{\theta}'_i) - \alpha \left(\mathbf{H}^i_{\bm{\theta} \bm{\theta}}\nabla_{\bm{\theta}'_i} f_i(\bm{\theta}'_i) + \mathbf{H}^i_{\bm{\theta}^* \bm{\theta}} \nabla_{\bm{(\bm{\theta}'_i)^*}} f_i(\bm{\theta}'_i)\right).
			\end{aligned}
		\end{equation}
		Then, $G_i(\theta^*)$ can be expressed as 
		\begin{equation}
			\begin{aligned} \label{40}
				G_i(\bm{\theta}^*) &=  \left(G_i(\theta)\right)^*\\
				&=\left( \nabla_{\bm{\theta}'_i} f_i(\bm{\theta}'_i) - \alpha \left(\mathbf{H}^i_{\bm{\theta} \bm{\theta}}\nabla_{\bm{\theta}'_i} f_i(\bm{\theta}'_i) + \mathbf{H}^i_{\bm{\theta}^* \bm{\theta}} \nabla_{\bm{(\bm{\theta}'_i)^*}} f_i(\bm{\theta}'_i)\right)\right)^*\\
				& = \nabla_{(\bm{\theta}'_i)^*} f_i(\bm{\theta}'_i) - \alpha \left(\mathbf{H}^i_{\bm{\theta} \bm{\theta}^* }\nabla_{\bm{\theta}'_i} f_i(\bm{\theta}'_i) + \mathbf{H}^i_{\bm{\theta}^* \bm{\theta}^*} \nabla_{\bm{(\bm{\theta}'_i)^*}} f_i(\bm{\theta}'_i)\right)
			\end{aligned}
		\end{equation}
		where the first and second equalities are because the loss function is real numbers. According to Wirtinger derivatives \cite{hjorungnes2011complex}, define conjugate coordinates $\mathbf{\phi} = \begin{bmatrix} \bm{\theta} \\ \bm{\theta}^* \end{bmatrix}$, and we have
		\begin{equation}
			\begin{aligned} \label{41}
				G_i(\bm{\phi}) &= \begin{bmatrix}  G_i(\bm{\theta}) \\ G_i(\bm{\theta}^*) \end{bmatrix} \\
				&= \begin{bmatrix} \nabla_{\bm{\theta}'_i} f_i(\bm{\theta}'_i) - \alpha \left(\mathbf{H}^i_{\bm{\theta} \bm{\theta}}\nabla_{\bm{\theta}'_i} f_i(\bm{\theta}'_i) + \mathbf{H}^i_{\bm{\theta}^* \bm{\theta}} \nabla_{\bm{(\bm{\theta}'_i)^*}} f_i(\bm{\theta}'_i)\right) \\ \nabla_{(\bm{\theta}'_i)^*} f_i(\bm{\theta}'_i) - \alpha \left(\mathbf{H}^i_{\bm{\theta} \bm{\theta}^* }\nabla_{\bm{\theta}'_i} f_i(\bm{\theta}'_i) + \mathbf{H}^i_{\bm{\theta}^* \bm{\theta}^*} \nabla_{\bm{(\bm{\theta}'_i)^*}} f_i(\bm{\theta}'_i)\right) \end{bmatrix}\\
				&= \begin{bmatrix} \nabla_{\bm{\theta}'_i} f_i(\bm{\theta}'_i)  \\ \nabla_{(\bm{\theta}'_i)^*} f_i(\bm{\theta}'_i)  \end{bmatrix} - \alpha \begin{bmatrix} \mathbf{H}^i_{\bm{\theta} \bm{\theta}} & \mathbf{H}^i_{\bm{\theta}^* \bm{\theta}} \\ \mathbf{H}^i_{\bm{\theta} \bm{\theta}^* } & \mathbf{H}^i_{\bm{\theta}^* \bm{\theta}^*} \end{bmatrix} \begin{bmatrix} \nabla_{\bm{\theta}'_i} f_i(\bm{\theta}'_i)  \\ \nabla_{(\bm{\theta}'_i)^*} f_i(\bm{\theta}'_i)  \end{bmatrix} \\
				& =  \nabla_{\bm{\phi}'_i} f_i(\bm{\theta}'_i) - \alpha \nabla_{\bm{\phi}}^2 \mathcal{L}_{S_i}(\bm{\theta}) \nabla_{\bm{\phi}'_i} f_i(\bm{\theta}'_i) \\
				& = \left(\mathbf{I} -  \alpha \nabla_{\bm{\phi}}^2 f_i(\bm{\theta}) \right) \nabla_{\bm{\phi}'_i} f_i(\bm{\theta}'_i)\\
				& = \left(\mathbf{I} -  \alpha \nabla_{\bm{\phi}}^2 f_i(\bm{\phi}) \right) \nabla_{\bm{\phi}'_i} f_i(\bm{\phi}'_i)
			\end{aligned}
		\end{equation}
		where the second equality follows \eqref{40} and \eqref{41}, the fourth equality is given by the definition of complex gradient and hessian, and the last equality follows Wirtinger derivatives. Obviously, $G_i(\bm{\phi})$ has the same form with the outer-loop update gradient in MAML. Thus, with the Assumptions 2-5 which extend the Assumptions in \cite{fallah2020convergence} to complex domain, we have similar derivation and conclusion as \cite{fallah2020convergence}. Since $\bm{\theta} = \left(\bm{\theta}^*\right)^*$ and $G_i(\bm{\theta}) = \left(G_i(\bm{\theta}^*)\right)^*$, updating meta-parameter $\bm{\theta}$ is equivalent to updating $\bm{\phi}$. Hence, Theorem \ref{theorem 1} holds.
		This completes the proof.
		
	\end{proof}
	
	\section{Toy Experiment of the Chain Rule}
	The complex chain rule is significant for the back-propagation of CDCVNN, especially in the presence of non-analytic functions, and we conduct a toy experiment to verify it. 
	The toy network is defined as $J=h_3(h_2(h_1(x)))$, where $h_1(x)=x^*$, $h_2(x)=x^2$, and $h_3(x)=|\exp{-x}|$. Figure \ref{toyexample} demonstrates that the complex chain rule is succeed in optimizing $J$ while the naive chain rule fail due to $\dfrac{\partial h_1(x)}{\partial x} = \dfrac{\partial x^*}{\partial x} = 0$. 
	
	\begin{figure}[htbp]
		\centering
		\includegraphics[width= 10 cm]{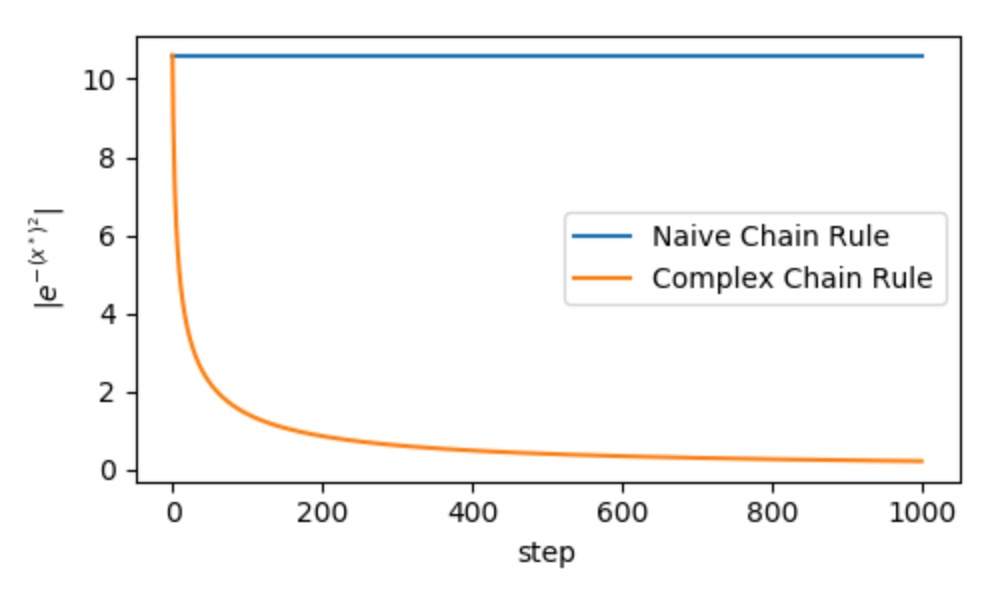}
		\caption{The comparison of back-propagation of the toy network $J$ via the complex chain rule and naive chain rule.}
		\label{toyexample}
	\end{figure}
	
	\section{Confusion Matrices}
	The class-confusion matrices for MAML and CAMEL are showed in Figure~\ref{fig4} and Figure~\ref{fig5} for 1 shot and 5 shot case, respectively. These tell the actual labels and predict labels of the testing samples. From the images we can find out that we will have better classification performance in 5 shot case than 1 shot case, and better performance with model CAMEL than MAML. The results also indicate that two models both perform well for the class ’CPFSK’ and are easy to get confused for the other 4 classes.  Additionally, CAMEL has much better performance for these remaining 4 classes.
	
	\begin{figure}[htbp]
		
		\subfigure 
		{
			\begin{minipage}{6.8cm}
				\centering          
				\includegraphics[scale=0.45]{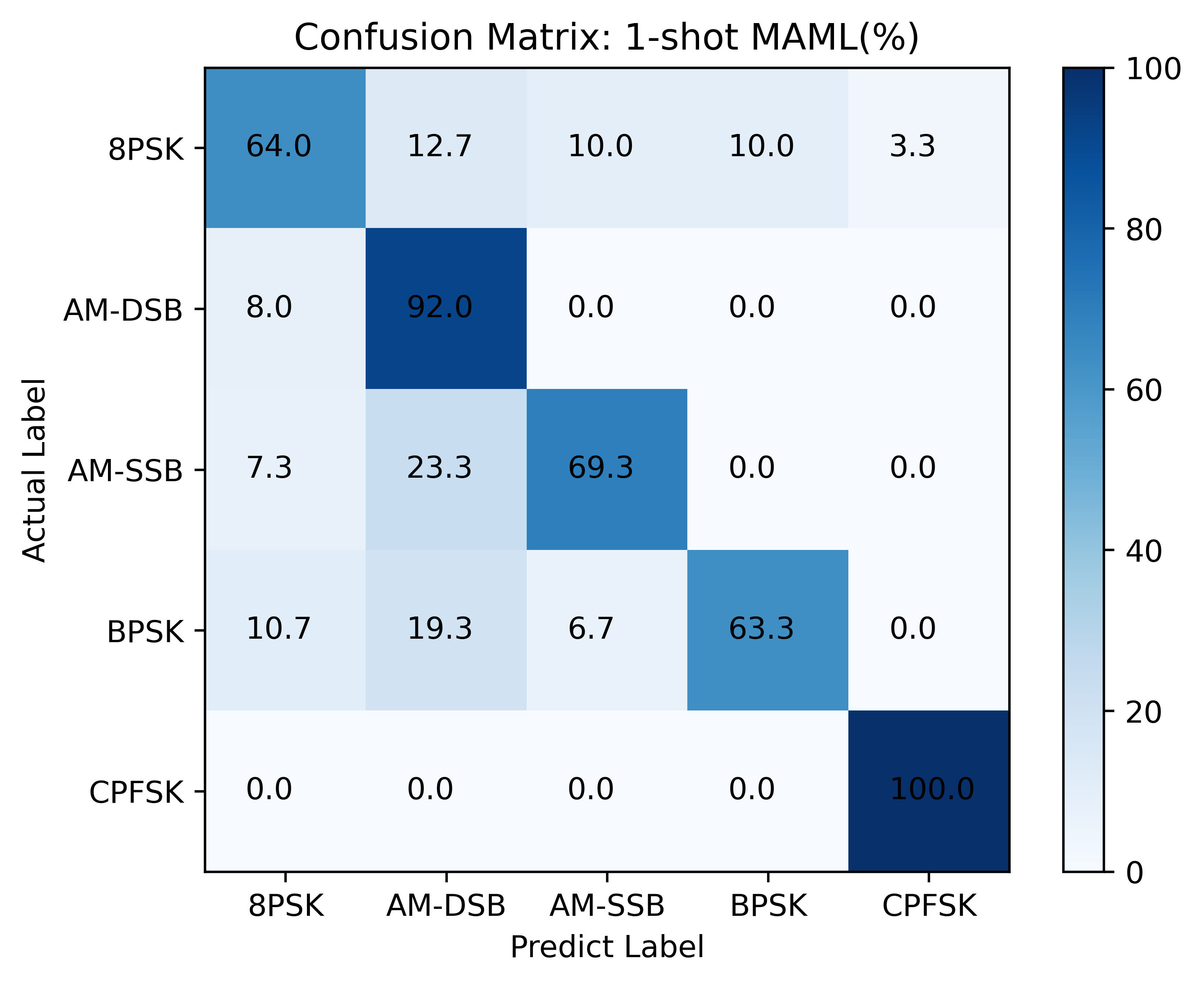}   
			\end{minipage}
		}
		\subfigure 
		{
			\begin{minipage}{6.8cm}
				\centering      
				\includegraphics[scale=0.45]{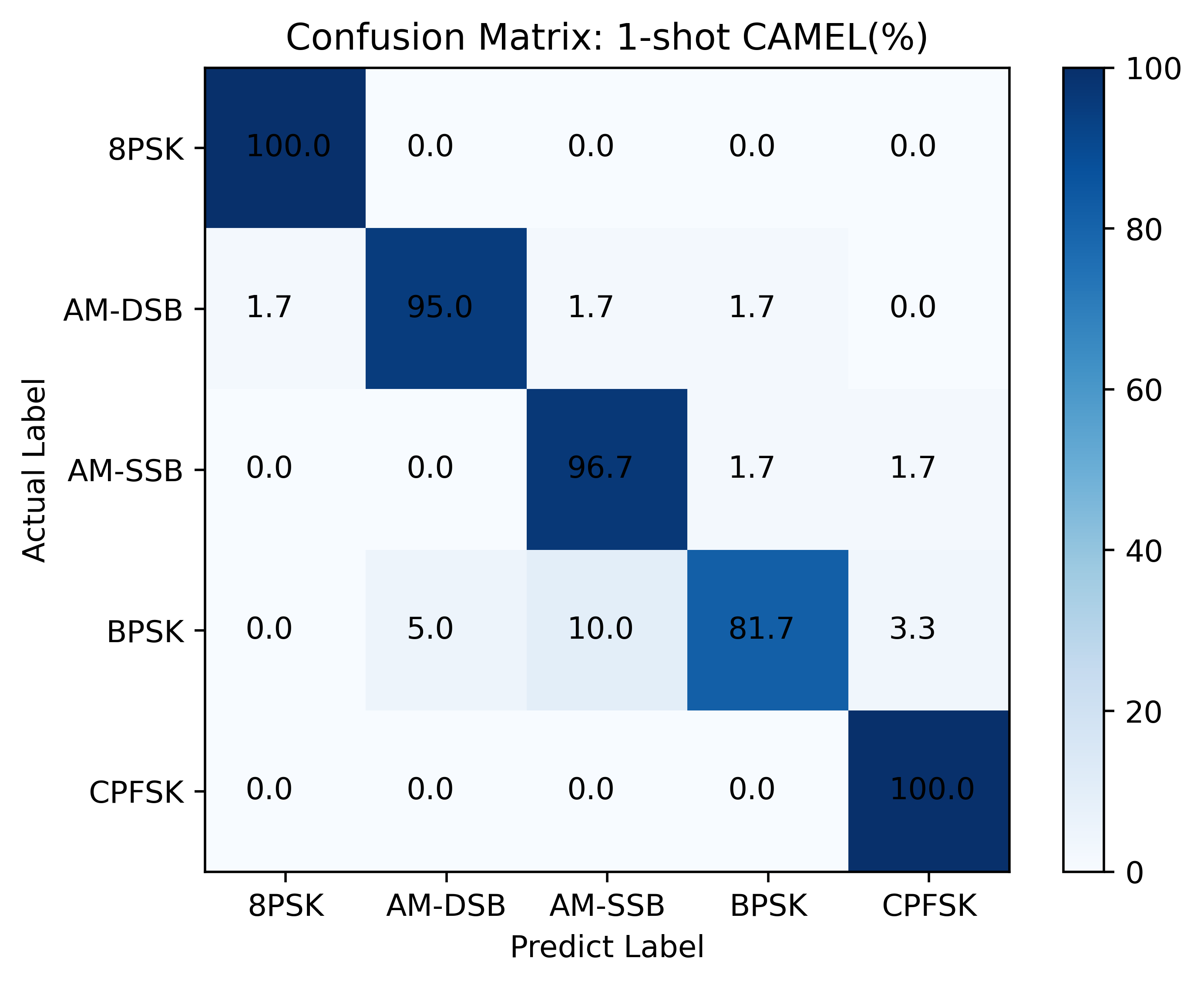}   
			\end{minipage}
		}
		
		\caption{Confusion Matrix: compare our model CAMEL with MAML in 1-shot case for classification tasks on the dataset RADIOML 2016.10A. The y axis shows the actual label of samples and the x axis shows the predict label. The data is presented in percentage. Left: Model-Agnostic Meta-Learning model. Right: Complex-valued  Attentional MEta Learner. The results tell that two models both perform well for the class 'CPFSK' and CAMEL has much better performance for other classes.} 
		\label{fig4}  
	\end{figure}
	
	\begin{figure}[htbp]
		
		\subfigure 
		{
			\begin{minipage}{6.8cm}
				\centering          
				\includegraphics[scale=0.45]{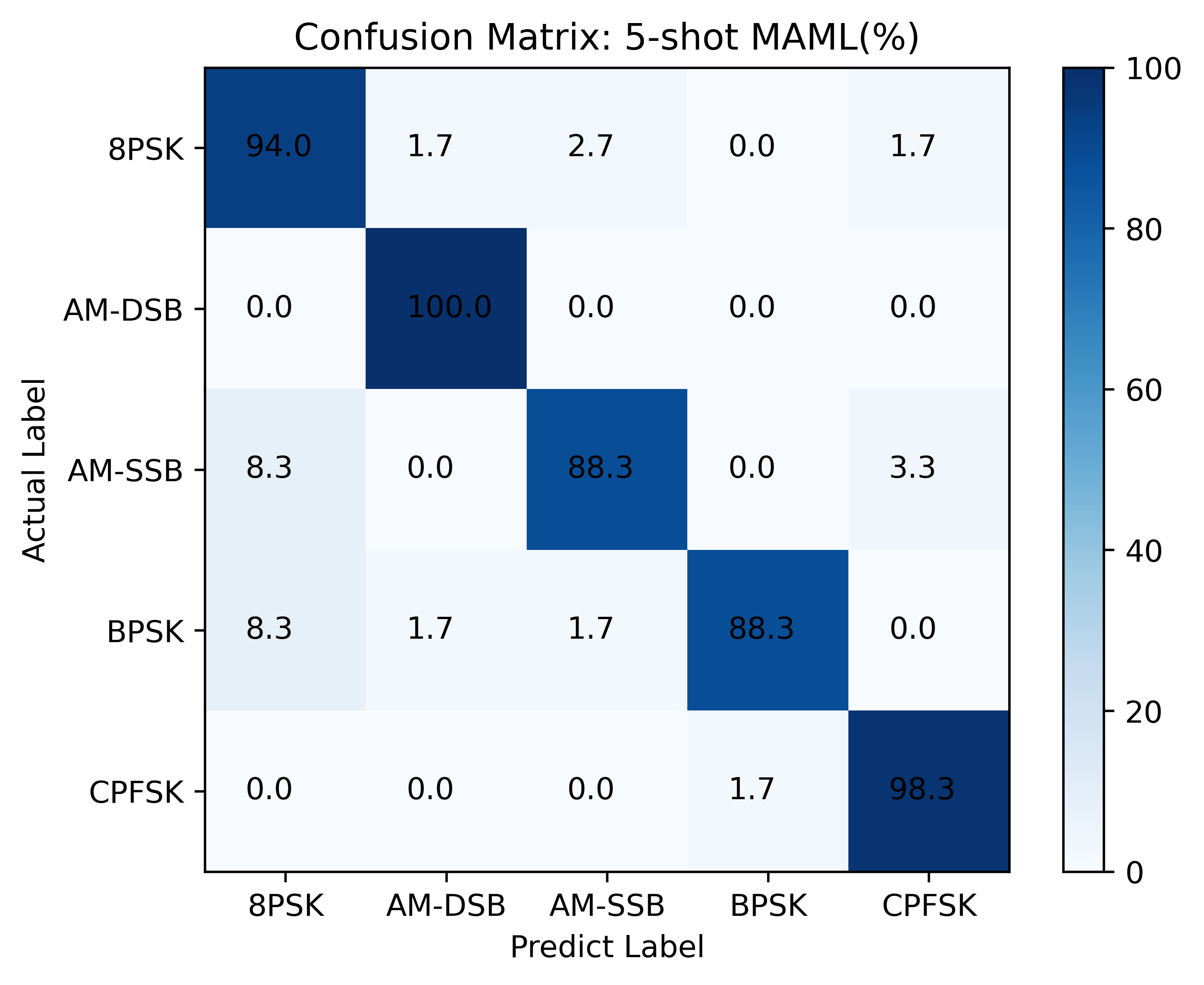}   
			\end{minipage}
		}
		\subfigure 
		{
			\begin{minipage}{6.8cm}
				\centering      
				\includegraphics[scale=0.45]{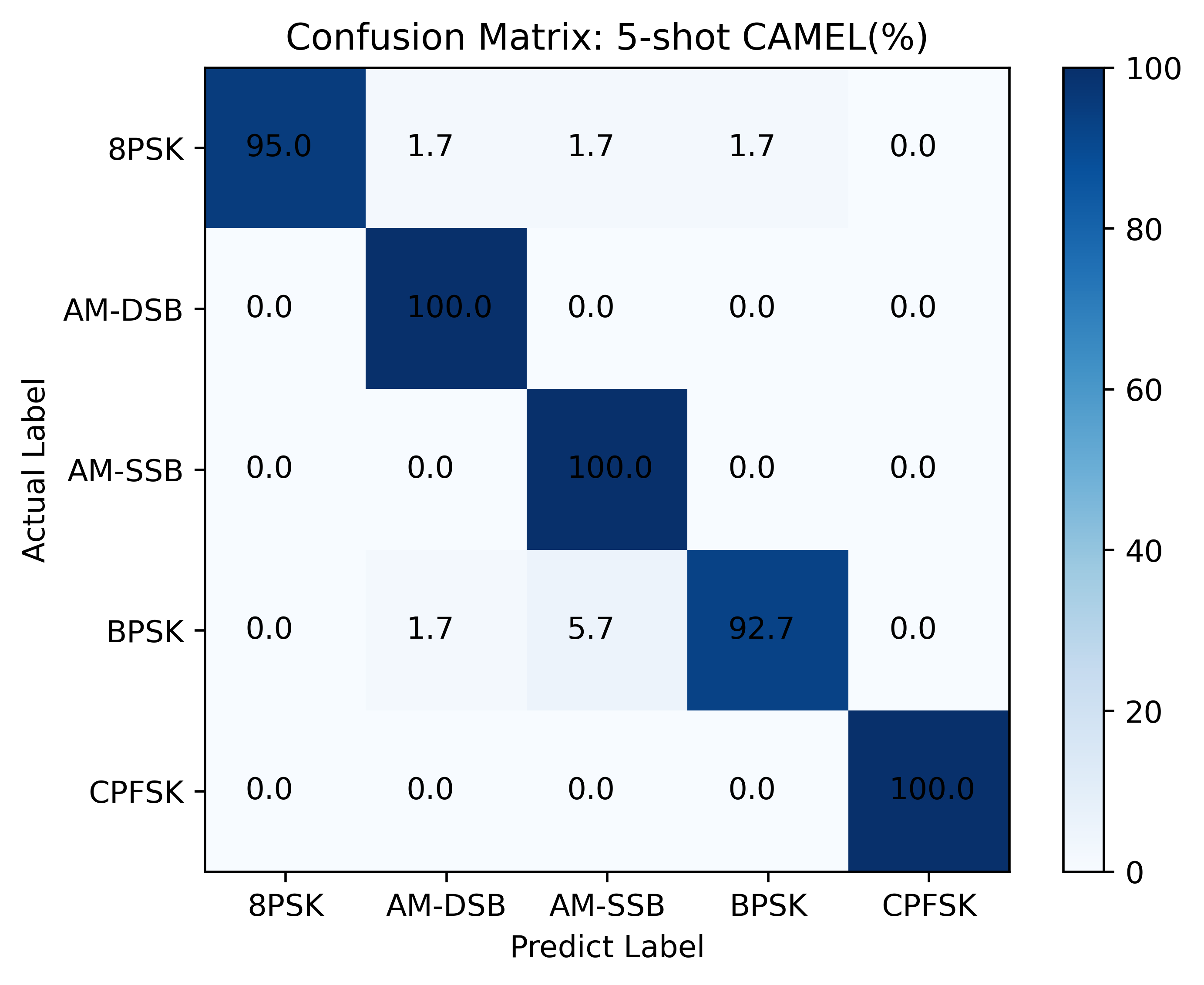}   
			\end{minipage}
		}
		
		\caption{Confusion Matrix: compare our model CAMEL with MAML in 5-shot case for classification tasks on the dataset RADIOML 2016.10A. The data is presented in percentage. Left: MAML. Right: CAMEL. From the result, two models are both able to fulfil the classification task well, and CAMEL does a better job.} 
		\label{fig5}  
	\end{figure}
	
\end{document}